\documentclass{article}

\PassOptionsToPackage{authoryear, round}{natbib}

\usepackage[final]{neurips_2022}

\usepackage[utf8]{inputenc} 
\usepackage[T1]{fontenc}    
\usepackage{hyperref}       
\usepackage{url}            
\usepackage{booktabs}       
\usepackage{amsfonts}       
\usepackage{nicefrac}       
\usepackage{microtype}      
\usepackage{xcolor}         

\usepackage{algorithm}
\usepackage[noend]{algpseudocode}

\usepackage{bm}
\usepackage{microtype}
\usepackage{graphicx}
\usepackage{subcaption}
\usepackage{amsmath,amsthm,amssymb}
\usepackage{wrapfig}
\usepackage{thmtools}
\usepackage{multirow}
\usepackage{listings}
\usepackage{caption}
\usepackage{array}
\usepackage{subcaption}
\usepackage[all]{xy}

\usepackage{enumitem}
\setlist{leftmargin=10mm}

\title{Efficient Meta Reinforcement Learning for Preference-based Fast Adaptation}

\author{%
  Zhizhou Ren$^{12}$, Anji Liu$^3$, Yitao Liang$^{45}$, Jian Peng$^{126}$, Jianzhu Ma$^6$ \\
  $^1$Helixon Ltd. $^2$University of Illinois at Urbana-Champaign \\
  $^3$University of California, Los Angeles \\
  $^4$Institute for Artificial Intelligence, Peking University\\
  $^5$Beijing Institute for General Artificial Intelligence\\
  $^6$Institute for AI Industry Research, Tsinghua University\\
  \texttt{zhizhour@helixon.com}, \texttt{liuanji@cs.ucla.edu} \\
  \texttt{yitaol@pku.edu.cn}, \texttt{jianpeng@illinois.edu} \\
  \texttt{majianzhu@tsinghua.edu.cn} \\
}

\begin{document}

\newcommand{\ie}{{\textit{i.e.}}}
\newcommand{\eg}{{\textit{e.g.}}}
\newcommand{\RenyiUlam}{R{\'e}nyi-Ulam}
\newcommand{\mismatch}{\mathcal{E}}
\newcommand{\BVol}{{\mathcal{BV}ol}_{\widehat Z}}

\newtheorem{theorem}{Theorem}
\newtheorem{definition}{Definition}
\newtheorem{corollary}{Corollary}
\newtheorem{lemma}{Lemma}
\newtheorem{fact}{Fact}
\newtheorem{implication}{Implication}
\newtheorem{proposition}{Proposition}
\newtheorem{assumption}{Assumption}

\maketitle

\begin{abstract}
	Learning new task-specific skills from a few trials is a fundamental challenge for artificial intelligence. Meta reinforcement learning (meta-RL) tackles this problem by learning transferable policies that support few-shot adaptation to unseen tasks. Despite recent advances in meta-RL, most existing methods require the access to the environmental reward function of new tasks to infer the task objective, which is not realistic in many practical applications. To bridge this gap, we study the problem of few-shot adaptation in the context of human-in-the-loop reinforcement learning. We develop a meta-RL algorithm that enables fast policy adaptation with preference-based feedback. The agent can adapt to new tasks by querying human's preference between behavior trajectories instead of using per-step numeric rewards. By extending techniques from information theory, our approach can design query sequences to maximize the information gain from human interactions while tolerating the inherent error of non-expert human oracle. In experiments, we extensively evaluate our method, \textit{Adaptation with Noisy OracLE} (ANOLE), on a variety of meta-RL benchmark tasks and demonstrate substantial improvement over baseline algorithms in terms of both feedback efficiency and error tolerance.
\end{abstract}

\section{Introduction}

Reinforcement learning (RL) has achieved great success in learning complex behaviors and strategies in a variety of sequential decision-making problems, including Atari games \citep{mnih2015human}, board game Go \citep{silver2016mastering}, MOBA games \citep{berner2019dota}, and real-time strategy games \citep{vinyals2019grandmaster}. However, most these breakthrough accomplishments are limited in simulation environments due to the sample inefficiency of RL algorithms. Training a policy using deep reinforcement learning usually requires millions of interaction samples with the environment, which is not practicable in real-world applications. One promising methodology towards breaking this practical barrier is \textit{meta reinforcement learning} \citep[meta-RL][]{finn2017model}. The goal of meta-RL is to enable fast policy adaptation to unseen tasks with a small amount of samples. Such an ability of few-shot adaptation is supported by meta-training on a suite of tasks drawn from a prior task distribution. Meta-RL algorithms can extract transferable knowledge from the meta-training experiences by exploiting the common structures among the prior training tasks. A series of recent works have been developed to improve the efficiency of meta-RL in several aspects, \eg, using off-policy techniques to improve the sample efficiency of meta-training \citep{rakelly2019efficient, fakoor2020meta}, and refining the exploration-exploitation trade-off during adaptation \citep{zintgraf2020varibad, liu2021decoupling}. There are also theoretical works studying the notion of few-shot adaptation and knowledge reuse in RL problems \citep{brunskill2014pac, wang2020global, chua2021provable}.

While recent advances remarkably improve the sample efficiency of meta-RL algorithms, little work has been done regarding the type of supervision signals adopted for policy adaptation. The adaptation procedure of most meta-RL algorithms is implemented by either fine-tuning policies using new trajectories \citep{finn2017model, fakoor2020meta} or inferring task objective from new reward functions \citep{duan2016rl, rakelly2019efficient, zintgraf2020varibad}, both of which require the access to the environmental reward function in new tasks to perform adaptation. Such a reward-driven adaptation protocol may become impracticable in many application scenarios. For instance, when a meta-trained robot is designed to provide customizable services for non-expert human users \citep{prewett2010managing}, it is not realistic for the meta-RL algorithm to obtain per-step reward signals directly from the user interface. The design of reward function is a long-lasting challenge for reinforcement learning \citep{sorg2010reward}, and there is no general solution to specify rewards for a particular goal \citep{amodei2016concrete, abel2021expressivity}. Existing techniques can support reward-free meta-training by constructing a diverse policy family through unsupervised reward generation \citep{gupta2020unsupervised}, and reward-free few-shot policy adaptation remains a challenge \citep{liu2020explore}.

In this paper, we study the interpolation of meta-RL and human-in-the-loop RL, towards expanding the applicability of meta-RL in practical scenarios without environmental rewards. We pursue the few-shot adaptation from human preferences \citep{furnkranz2012preference}, where the agent infers the objective of new tasks by querying a human oracle to compare the preference between pairs of behavior trajectories. Such a preference-based supervision is more intuitive than numeric rewards for human users to instruct the meta policy. \eg, the human user can watch the videos of two behavior trajectories and label a preference order to express the task objective \citep{christiano2017deep}. The primary goal of such a preference-based adaptation setting is to suit the user's preference with only a few preference queries to the human oracle. Minimizing the burden of interactions is a central objective of human-in-the-loop learning. In addition, we require the adaptation algorithm to be robust when the oracle feedback carries some noises. The human preference is known to have some extent of inconsistency \citep{loewenstein1992anomalies}, and the human user may also make unintentional mistakes when responding to preference queries. The tolerance to feedback error is an important evaluation metric for preference-based reinforcement learning \citep{lee2021b}.

To develop an efficient and robust adaptation algorithm, we draw a connection with a classical problem called \textit{\RenyiUlam's game} \citep{renyi1961problem, ulam1976adventures} from information theory. We model the preference-based task inference as a noisy-channel communication problem, where the agent communicates with the human oracle and infers the task objective from the noisy binary feedback. Through this problem transformation, we extend an information quantification called \textit{Berlekamp's volume} \citep{berlekamp1968block} to measure the uncertainty of noisy preference feedback. This powerful toolkit enables us to design the query contents to maximize the information gain from the noisy preference oracle. We implement our few-shot adaptation algorithm, called \textit{Adaptation with Noisy OracLE} (ANOLE), upon an advanced framework of probabilistic meta-RL \citep{rakelly2019efficient} and conduct extensive evaluation on a suite of Meta-RL benchmark tasks. The experiment results show that our method can effectively infer the task objective from limited interactions with the preference oracle and, meanwhile, demonstrate robust error tolerance against the noisy feedback.

\section{Preliminaries}

\subsection{Problem Formulation of Meta Reinforcement Learning} \label{sec:problem-formulation}

\paragraph{Task Distribution.}
The same as previous meta-RL settings \citep{rakelly2019efficient}, we consider a task distribution $p(\mathcal{T})$ where each task instance $\mathcal{T}$ induces a \textit{Markov Decision Process} (MDP). We use a tuple ${\mathcal{M}}_{\mathcal{T}}=\langle \mathcal{S}, \mathcal{A}, T, R_{\mathcal{T}}\rangle$ to denote the MDP specified by task $\mathcal{T}$. In this paper, we assume task $\mathcal{T}$ does not vary the environment configuration, including state space $\mathcal{S}$, action space $\mathcal{A}$, and transition function $T(s'\mid s,a)$. Only reward function $R_{\mathcal{T}}(\cdot\mid s,a)$ conditions on task $\mathcal{T}$, which defines the agent's objectives for task-specific goals. The meta-RL algorithm is allowed to sample a suite of training tasks $\{{\mathcal{T}}_i\}_{i=1}^N$ from the task distribution $p(\mathcal{T})$ to support meta-training. In the meta-testing phase, a new set of tasks are drawn from $p(\mathcal{T})$ to evaluate the adaptation performance.

\paragraph{Preference-based Adaptation.}
We study the few-shot adaptation with preference-based feedback, in which the agent interacts with a black-box preference oracle $\Omega_{\mathcal{T}}$ rather than directly receiving task-specific rewards from the environment $\mathcal{M}_{\mathcal{T}}$ for meta-testing adaptation. More specifically, the agent can query the preference order of a pair of trajectories $\langle\tau^{(1)},\tau^{(2)}\rangle$ through the black-box oracle $\Omega_{\mathcal{T}}$. Each trajectory is a sequence of observed states and agent actions, denoted by $\tau=\langle s_0, a_0, s_1, a_1, \cdots, s_L\rangle$ where $L$ is the trajectory length. For each query trajectory pair $\langle\tau^{(1)},\tau^{(2)}\rangle$, the preference oracle $\Omega_{\mathcal{T}}$ would return either $\tau^{(1)}\succ\tau^{(2)}$ or $\tau^{(1)}\prec\tau^{(2)}$ according to the task specification $\mathcal{T}$, where $\tau^{(1)}\succ\tau^{(2)}$ means the oracle prefers trajectory $\tau^{(1)}$ to trajectory $\tau^{(2)}$ under the context of task $\mathcal{T}$. When the preference orders of $\tau^{(1)}$ and $\tau^{(2)}$ are equal, both returns are valid. The preference-based adaptation is a weak-supervision setting in comparison with previous meta-RL works using per-step reward signals for few-shot adaptation \citep{finn2017model, rakelly2019efficient}. The purpose of this adaptation setting is to simulate the practical scenarios with human-in-the-loop supervision \citep{wirth2017survey, christiano2017deep}. We consider two aspects to evaluate the ability of an adaptation algorithm:
\begin{itemize}
	\item \textbf{Feedback Efficiency.} The central goal of meta-RL is conducting fast adaptation to unseen tasks with a few task-specific feedback. We consider the adaptation efficiency in terms of the number of preference queries to oracle $\Omega_{\mathcal{T}}$. This objective expresses the practical demand that we aim to reduce the burden of preference oracle since human feedback is expensive.
	\item \textbf{Error Tolerance.} Another performance metric is on the robustness against the noisy oracle feedback. The preference oracle $\Omega_{\mathcal{T}}$ may carry errors in practice, \eg, the human oracle may misunderstand the query message and give wrong feedback. Tolerating such oracle errors is a practical challenge for preference-based reinforcement learning.
\end{itemize}

\subsection{Meta Reinforcement Learning with Probabilistic Task Embedding}

\paragraph{Latent Task Embedding.}
We follow the algorithmic framework of \textit{Probabilistic Embeddings for Actor-critic RL} \citep[PEARL;][]{rakelly2019efficient}. The task specification $\mathcal{T}$ is modeled by a latent task variable (or latent task embedding) $z\in\mathcal{Z}=\mathbb{R}^d$ where $d$ denotes the dimension of the latent space. With this formulation, the overall paradigm of the meta-training procedure resembles a multi-task RL algorithm. Both policy $\pi(a|s;z)$ and value function $Q(s,a;z)$ condition on the latent task variable $z$ so that the representation of $z$ can be end-to-end learned with the RL objective to distinguish different task specifications. During meta-testing, the adaptation is performed in the low-dimensional task embedding space rather than the high-dimensional parameter space.

\paragraph{Adaptation via Probabilistic Inference.}
To infer the task embedding $z$ from the latent task space, PEARL trains an inference network (or context encoder) $q(z|\mathbf{c})$ where $\mathbf{c}$ is the context information including agent actions, observations, and rewards. The output of $q(z|\mathbf{c})$ is probabilistic, \ie, the agent has a probabilistic belief over the latent task space $\mathcal{Z}$ based on its observations and received rewards. We use $q(z)$ to denote the prior distribution for $\mathbf{c}=\emptyset$. The adaptation protocol of PEARL follows the framework of \textit{posterior sampling} \citep{strens2000bayesian}. The agent continually updates its belief by interacting with the environment and refine its policy according to the belief state. This algorithmic framework is generalized from Bayesian inference \citep{thompson1933likelihood} and has solid background in reinforcement learning theory \citep{agrawal2012analysis, osband2013more, russo2014learning}. However, some recent works show that the empirical performance of neural inference networks highly rely on the access to a dense reward function \citep{zhang2021metacure, hua2021hmrl}. When the task-specific reward signals are sparsely distributed along the agent trajectory, the task inference given by context encoder $q(z|\mathbf{c})$ would suffer from low sample efficiency and cannot accurately decode task specification. This issue may worsen in our adaptation setting since only trajectory-wise preference comparisons are available to the agent. It motivates us to explore new methodology for few-shot adaptation beyond classical approaches based on posterior sampling.

\section{Preference-based Fast Adaptation with A Noisy Oracle}

In this section, we will introduce our method, \textit{Adaptation with Noisy OracLE} (ANOLE), a novel task inference algorithm for preference-based fast adaptation. The goal of our approach is to achieve both high feedback efficiency and error tolerance.

\subsection{Connecting Preference-based Task Inference with \RenyiUlam's Game}

To give an information-theoretic view on the task inference problem with preference feedback, we connect our problem setting with a classical problem called \textit{\RenyiUlam's Game} \citep{renyi1961problem, ulam1976adventures} from information theory to study the interactive learning procedure with a noisy oracle.

\begin{definition}[\RenyiUlam's Game] \label{def:ulam_game}
	There are two players, called A and B, participating the game. Player A thinks of something in a predefined element universe, and player B would like to guess it. To extract information, player B can ask some questions to player A, and the answers to these questions are restricted to "yes/no". A given percentage of \textbf{player A's answers can be wrong}, which \textbf{requires player B's question strategy to be error-tolerant}.
\end{definition}

In the literature of information theory, \RenyiUlam's game specified in Definition~\ref{def:ulam_game} is developed to study the error-tolerant communication protocol for noisy channels \citep{shannon1956zero, renyi1984diary}. Most previous works on \RenyiUlam's game focus on the error tolerance of player B's question strategy, \ie, how to design the question sequence to maximize the information gain from the noisy feedback \citep{pelc2002searching}. In this paper, we consider the online setting of \RenyiUlam's game, where player B is allowed to continually propose queries based on previous feedback.

We draw a connection between \RenyiUlam's game and preference-based task inference. In the context of preference-based meta adaptation, the task inference algorithm corresponds to the questioner player B, and the preference oracle $\Omega$ plays the role of responder player A. The preference feedback given by oracle $\Omega$ is a binary signal regarding the comparison between two trajectories, which has the same form as the "yes/no" feedback in \RenyiUlam's game. The goal of the task inference algorithm is to search for the true task specification in the task space using minimum number of preference queries while tolerating the errors in oracle feedback. The similarity in problem structures motivates us to extend techniques from \RenyiUlam's game to preference-based task inference.

\subsection{An Algorithmic Framework for Preference-based Fast Adaptation}

In this section, we discuss how we transform the preference-based task inference problem to \RenyiUlam's game and introduce the basic algorithmic framework of our approach to perform few-shot adaptation with a preference oracle.

\paragraph{Transformation to \RenyiUlam's Game.}
The key step of the problem transformation is establishing the correspondence between the preference query in preference-based task inference and the natural-language-based questions in \RenyiUlam's game. In classical solutions of \RenyiUlam's game, a general format of player B's questions is to ask whether the element in player A's mind belongs to a certain subset of the element universe \citep{pelc2002searching}, whereas the counterpart of question format in preference-based task inference is restricted to querying whether the oracle prefers one trajectory to another. To bridge this gap, we use a model-based approach to connect the oracle preference feedback with the latent space of task embeddings. We train a preference predictor $f_\psi(\cdot;z)$ that predicts the oracle preference according to the task embedding $z$. This preference predictor can transform each oracle preference feedback to a separation on the latent task space, \ie, the preference prediction $f_\psi(\cdot;z)$ given by task embeddings in a subspace of $z$ can match the oracle feedback, and the task embeddings in the complementary subspace lead to wrong predictions. Through this transformation, the task inference algorithm can convert the binary preference feedback to the assessment of a subspace of latent task embeddings, which works in the same mechanism as previous solutions to \RenyiUlam's game.

To realize the problem transformation,  we consider a simple and direct implementation of the preference predictor. We train $f_\psi(\cdot;z)$ on the meta-training tasks by optimizing the preference loss:
\begin{align}
	\mathcal{L}^{\text{Pref}}(\psi) = \mathop{\mathbb{E}}\left[ D_{\text{KL}}\left(\left.\mathbb{I}\bigl[\tau^{(1)}\succ\tau^{(2)}\mid\mathcal{T}\bigr] ~\right\|~ f_\psi\bigl(\tau^{(1)}\succ\tau^{(2)};z\bigr)\right) \right],
\end{align}
where $\psi$ denotes the parameters of the preference predictor, $D_{\text{KL}}(\cdot\|\cdot)$ denotes the Kullback–Leibler divergence, and $\mathbb{I}[\tau^{(1)}\succ\tau^{(2)}\mid\mathcal{T}]$ is the ground-truth preference order specified by the task specification $\mathcal{T}$ (\eg, specified by the reward function on training tasks). The trajectory pair $\langle\tau^{(1)},\tau^{(2)}\rangle$ is drawn from the experience buffer, and $z$ is the task embedding vector encoding $\mathcal{T}$. More implementation details are included in Appendix~\ref{apx:our_implementation}.

\paragraph{Basic Algorithmic Framework.}
To facilitate discussions, we introduce some notations to model the interaction with the preference oracle $\Omega_{\mathcal{T}}$. Suppose the adaptation budget supports $K$ online preference queries to oracle $\Omega_{\mathcal{T}}$, which divides this interactive procedure into $K$ rounds. We define the notation of query context set as Definition~\ref{def:query_set} to represent the context information extracted from the preference queries during the online oracle interaction.

\begin{definition}[Query Context Set] \label{def:query_set}
	The query context set $\mathcal{Q}_k=\{\langle\tau_j^{(1)}\succ\tau_j^{(2)}\rangle\}_{j=1}^k$ denotes the set of preference queries completed at the first $k$ rounds, in which the task inference protocol queries the preference order between the trajectory pair $\langle\tau_j^{(1)}, \tau_j^{(2)}\rangle$ at the $j$th round. To simplify the notations, the trajectory pair $\langle\tau_j^{(1)},\tau_j^{(2)}\rangle$ are relabeled according to the oracle feedback so that the preference order given by oracle $\Omega$ is $\tau_j^{(1)}\succ \tau_j^{(2)}$ for any $1\leq j\leq k$.
\end{definition}

The query context set $\mathcal{Q}_k$ concludes the context information obtained from the oracle $\Omega$ in the first $k$ rounds. After completing the query at round $k$, the task inference algorithm needs to decide the next-round query trajectory pair $\langle\tau_{k+1}^{(1)}, \tau_{k+1}^{(2)}\rangle$ based on the context information stored in $\mathcal{Q}_k$. By leveraging the model-based problem transformation to \RenyiUlam's game, we can assess the quality of a task embedding $z$ by counting the number of mismatches with respect to oracle preferences, denoted by $\mismatch(z;\mathcal{Q}_k)$:
\begin{align}
	\mismatch(z;\mathcal{Q}_k) = \sum_{(\tau^{(1)}\succ\tau^{(2)})\in\mathcal{Q}_k} \mathbb{I}\left[f_\psi(\tau^{(1)}\succ\tau^{(2)}; z) < f_\psi(\tau^{(2)}\succ\tau^{(1)}; z)\right].
\end{align}

The overall algorithmic framework of our preference-based few-shot adaptation method, \textit{Adaptation with Noisy OracLE} (ANOLE), is summarized in Algorithm~\ref{alg:framework}.

\begin{algorithm}[h]
	\caption{Adaptation with Noisy OracLE (ANOLE)} \label{alg:framework}
	\begin{algorithmic}[1]
		\State {\bfseries input:} a preference oracle $\Omega_{\mathcal{T}}$, the budget of oracle queries $K$
		
		\hspace{0.095in} a prior task distribution $q(z)$, the size of candidate pool $M$
		
		\hspace{0.095in} a query generation protocol $\textsc{GenerateQuery}(\cdot)$
		\State $\widehat Z\gets \{z_j\sim q(z)\}_{j=1}^M$ \Comment{sample a candidate pool} \label{alg:sample_pool}
		\State $\mathcal{Q}_0\gets \emptyset$ \Comment{initialize context information}
		\For{$k=1$ {\bfseries to} $K$}
			\State $\langle\tau_k^{(1)},\tau_k^{(2)}\rangle\gets \textsc{GenerateQuery}(\widehat Z, \mathcal{Q}_{k-1})$ \Comment{\textbf{critical step:} query generation (Eq.~\eqref{eq:query_generation})} \label{alg:query_generation}
			\State $\langle\tau_k^{(u)}\succ \tau_k^{(v)}\rangle\gets \Omega_{\mathcal{T}}(\tau_j^{(1)},\tau_j^{(2)})$ where $u,v\in\{1,2\}$ \Comment{request oracle feedback}
			\State $\mathcal{Q}_k\gets \mathcal{Q}_{k-1}\cup \{\langle\tau_k^{(u)}\succ \tau_k^{(v)}\rangle\}$ \Comment{update context information}
		\EndFor
		\State {\bfseries return:} $\arg\min_{z\in\widehat Z} \mismatch(z;\mathcal{Q}_K)$ \label{alg:return}
	\end{algorithmic}
\end{algorithm}

The first step of our task inference algorithm is sampling a candidate pool $\widehat Z$ of latent task embeddings from the prior distribution $q(z)$. Then we perform $K$ rounds of candidate selection by querying the preference oracle $\Omega_{\mathcal{T}}$. The design of the query generation protocol $\textsc{GenerateQuery}(\cdot)$ is a critical component of our task inference algorithm and will be introduced in section~\ref{sec:query_generation}. The final decision would be the task embedding with minimum mismatch with the oracle preference, \ie, $\arg\min_{z\in\widehat Z} \mismatch(z;\mathcal{Q}_K)$.

\subsection{Error-Tolerant Task Inference for Noisy Preference Oracle} \label{sec:query_generation}

The problem transformation to \RenyiUlam's game enables us to leverage techniques from information theory to develop query strategy with both high feedback efficiency and error tolerance.

\paragraph{Binary-Search Paradigm.}
The basic idea is conducting a binary-search-like protocol to leverage the binary structure of preference feedback: \textit{After each round of oracle interaction, we shrink the candidate pool $\widehat Z$ by removing those task embeddings leading to wrong preference predictions $f_\psi(\cdot;z)$.} An ideal implementation of such a binary-search protocol with noiseless feedback is expected to roughly eliminate half of candidates using each single oracle preference feedback, which achieves the information-theoretic lower bound of interaction costs. In practice, we pursue to handle noisy feedback, since both the preference oracle $\Omega_{\mathcal{T}}$ and the preference predictor $f_\psi(\cdot;z)$ may carry errors. An error-tolerant binary-search protocol requires to establish an information quantification (\eg, an uncertainty metric) to evaluate the information gain of each noisy oracle feedback. The goal of oracle interaction is to rapidly reduce the uncertainty of task inference rather than simply eliminate the erroneous candidates. In this paper, we extend an information-theoretic tool called \textit{Berlekamp's volume} \citep{berlekamp1968block} to develop such an uncertainty quantification.

\paragraph{Berlekamp's Volume.}
One classical tool to deal with erroneous information in search problems is \textit{Berlekamp's volume}, which is first proposed by \citet{berlekamp1968block} and has been explored by subsequent works in numerous variants of \RenyiUlam's game \citep{rivest1980coping, pelc1987solution, lawler1995algorithm, aigner1996searching, cicalese2007perfect}. The primary purpose of Berlekamp's volume is to mimic the notion of Shannon entropy from information theory \citep{shannon1948mathematical} and specialize in the applications to noisy-channel communication \citep{shannon1956zero} and error-tolerant search \citep{renyi1961problem}. We refer readers to \citet{pelc2002searching} for a comprehensive literature review of the applications of Berlekamp's Volume and the solutions to \RenyiUlam's game.

In Definition~\ref{def:volume}, we rearrange the definition of Berlekamp's volume to suit the formulation of preference-based learning.

\begin{definition}[Berlekamp's Volume] \label{def:volume}
	Suppose the budget supports $K$ oracle queries in total, and the oracle may have at most $K_{E}$ incorrect feedback among these queries. Berlekamp's volume of a query context set $\mathcal{Q}_k$ is defined as follows:
	\begin{align} \label{eq:volume_def}
		\BVol(\mathcal{Q}_k) = \sum_{z\in\widehat Z} vol_z(\mathcal{Q}_k),
		\qquad
		vol_z(\mathcal{Q}_k) = \sum_{\ell=0}^{K_{E}-\mismatch(z;\mathcal{Q}_{k-1})} \binom{K-k}{\ell},
	\end{align}
	where $\binom{K-k}{\ell}$ denotes the binomial coefficient.
\end{definition}

As stated in Definition~\ref{def:volume}, the configuration of Berlekamp's volume has two hyper-parameters: the total number of queries $K$, and the maximum number of erroneous feedback $K_{E}$ within all queries. This error mode refers to \citet{berlekamp1968block}'s noisy-channel model. Berlekamp's volume is a tool for designing robust query strategy to guarantee the tolerance of bounded number of feedback errors.

\paragraph{What type of uncertainty does Berlekamp's volume characterize?}
Berlekamp's volume measures the uncertainty of the unknown latent task variable $z$ together with the randomness carried by the noisy oracle feedback. To give a more concrete view, we first show how the value of Berlekamp's volume $\BVol(\mathcal{Q}_k)$ changes when receiving the feedback of a new preference query. Depending on the binary oracle feedback for the query trajectory pair $(\tau_{k}^{(1)}, \tau_{k}^{(2)})$, the query context set may be updated to two possible statuses:
\begin{align}
	\mathcal{Q}^{(1)\succ(2)}_{k} = \mathcal{Q}_{k-1}\cup \left\{(\tau_{k}^{(1)}\succ \tau_{k}^{(2)})\right\}
	\quad\text{and}\quad
	\mathcal{Q}^{(2)\succ(1)}_{k} = \mathcal{Q}_{k-1}\cup \left\{(\tau_{k}^{(2)}\succ \tau_{k}^{(1)})\right\},
\end{align}
where $\mathcal{Q}^{(u)\succ(v)}_{k}$ denotes the updated status in the case of receiving oracle feedback $\tau_{k}^{(u)}\succ \tau_{k}^{(v)}$. The relation between $\BVol(\mathcal{Q}_{k-1})$ and $\BVol(\mathcal{Q}_{k})$ is characterized by Proposition~\ref{prop:volume_conservation}, which is the foundation of Berlekamp's volume for developing error-tolerant algorithms.
\begin{restatable}[Volume Conservation Law]{proposition}{VolumeProposition} \label{prop:volume_conservation}
	For any query context set $\mathcal{Q}_{k-1}$ and arbitrary query trajectory pair $(\tau_{k}^{(1)}, \tau_{k}^{(2)})$, the relation between $\BVol(\mathcal{Q}_{k-1})$ and $\BVol(\mathcal{Q}_{k})$ satisfies
	\begin{align}
		\BVol(\mathcal{Q}_{k-1}) = \BVol(\mathcal{Q}^{(1)\succ(2)}_{k}) + \BVol(\mathcal{Q}^{(2)\succ(1)}_{k}).
	\end{align}
\end{restatable}
The proofs of all statements presented in this section are deferred to Appendix~\ref{apx:proof}. As shown by Proposition~\ref{prop:volume_conservation}, each preference query would partition the volume $\BVol(\mathcal{Q}_{k-1})$ into two subsequent branches, $\BVol(\mathcal{Q}^{(1)\succ(2)}_{k})$ and $\BVol(\mathcal{Q}^{(2)\succ(1)}_{k})$. The selection of preference queries does not alter the volume sum of subsequent branches. Since the values of $\BVol(\cdot)$ are non-negative integers, the volume is monotonically decreased with the online query procedure. When the volume is eliminated to the unit value, the selection of task embedding with minimum mismatches $\mismatch(z;\mathcal{Q}_k)$ would become deterministic (see Proposition~\ref{prop:unit_leaf}).

\begin{restatable}[Unit of Volume]{proposition}{UnitProposition} \label{prop:unit_leaf}
	Given a query context set $\mathcal{Q}_k$ with $\BVol(\mathcal{Q}_k)=1$, there exists exactly one task embedding candidate $z\in\widehat Z$ satisfying $\mismatch(z;\mathcal{Q}_k)\leq K_{E}$.
\end{restatable}

We can represent the preference-based task inference protocol by a \textit{decision tree}, where each $\mathcal{Q}_k$ with $\BVol(\mathcal{Q}_k)=1$ corresponds to a \textit{leaf node} and each preference query is a \textit{decision rule}. The value of Berlekamp's volume $\BVol(\mathcal{Q}_k)$ corresponds to the number of leaf nodes remaining in the subtree rooted with context set $\mathcal{Q}_k$. The value of $z$-conditioned volume $vol_z(\mathcal{Q}_k)$ counts the number of leaf nodes with $z$ as the final decision. Each path from an ancestor node to a leaf node can be mapped to a valid feedback sequence that does not violate $\mismatch(z;\mathcal{Q}_K)\leq K_E$. From this perspective, Berlekamp's volume quantifies the uncertainty of task inference by counting the number of valid feedback sequences for the incoming queries.

\paragraph{Error-Tolerant Query Strategy.}
Given Berlekamp's volume as the uncertainty quantification, the query strategy can be constructed directly by maximizing the worst-case uncertainty reduction:
\begin{align} \label{eq:query_generation}
	\textsc{GenerateQuery}(\widehat Z, \mathcal{Q}_{k-1}) = \mathop{\arg\min}_{\tau_k^{(1)}, \tau_k^{(2)}} \biggl(\max\Bigl\{ \BVol(\mathcal{Q}^{(1)\succ(2)}_{k}), ~\BVol(\mathcal{Q}^{(2)\succ(1)}_{k}) \Bigr\}\biggr),
\end{align}
where $\textsc{GenerateQuery}(\cdot)$ refers to the query generation step at line~\ref{alg:query_generation} of Algorithm~\ref{alg:framework}. This design follows the principle of binary search. If the Berlekamp's volumes of two subsequent branches are well balanced, the uncertainty can be exponentially reduced no matter which feedback the oracle responds. In our implementation, the $\arg\min$ operator in Eq.~\eqref{eq:query_generation} is approximated by sampling a mini-batch of trajectory pairs from the experience buffer. The query content is determined by finding the best trajectory pair within the sampled mini-batch. More implementation details are included in Appendix~\ref{apx:our_implementation}.


\section{Experiments}

In this section, we investigate the empirical performance of ANOLE on a suite of Meta-RL benchmark tasks. We compare our method with simple preference-based adaptation strategies and conduct several ablation studies to demonstrate the effectiveness of our algorithmic designs. 

\subsection{Experiment Setup}

\paragraph{Experiment Setting.}
We adopt six meta-RL benchmark tasks created by \citet{rothfuss2019promp}, which are widely used by meta-RL works to evaluate the performance of few-shot policy adaptation \citep{rakelly2019efficient, zintgraf2020varibad, fakoor2020meta}. These environment settings consider four ways to vary the task specification $\mathcal{T}$: forward/backward (-Fwd-Back), random target velocity (-Rand-Vel), random target direction (-Rand-Dir), and random target location (-Rand-Goal). We simulate the preference oracle $\Omega_{\mathcal{T}}$ by comparing the ground-truth trajectory return given by the MuJoCo-based environment simulator. The adaptation protocol cannot observe these environmental rewards during meta-testing and can only query the preference oracle $\Omega_{\mathcal{T}}$ to extract the task information. In addition, we impose a random noisy perturbation on the oracle feedback. Each binary feedback would be flipped with probability $\epsilon$. We consider such independently distributed errors rather than the bounded-number error mode to evaluate the empirical performance, since it is more realistic to simulate human's unintended errors. A detailed description of the experiment setting is included in Appendix~\ref{apx:experiment_setting}.

\paragraph{Implementation of ANOLE.}
Note that ANOLE is an adaptation module and can be built upon any meta-RL or multi-task RL algorithms with latent policy encoding \citep{hausman2018learning, eysenbach2019diversity, pong2020skew, lynch2020learning, gupta2020unsupervised}. To align with the baseline algorithms, we implement a meta-training procedure similar to PEARL \citep{rakelly2019efficient}. The same as PEARL, our policy optimization module extends soft actor-critic \citep[SAC;][]{haarnoja2018soft} with the latent task embedding $z$, \ie, the policy and value functions are represented by $\pi_\theta(a\mid s;z)$ and $Q_\phi(s,a;z)$ where $\theta$ and $\phi$ denote the parameters. One difference from PEARL is the removal of the inference network, since it is no longer used in meta-testing. Instead, we set up a bunch of trainable latent task variables $\{z_i\}_{i=1}^N$ to learn the multi-task policy. More implementation details are included in Appendix~\ref{apx:our_implementation}. The source code of our ANOLE implementation and experiment scripts are available at \url{https://github.com/Stilwell-Git/Adaptation-with-Noisy-OracLE}.

\paragraph{Baseline Algorithms.}
We compare the performance with four baseline adaptation strategies. Two of these baselines are built upon the same meta-training pre-processing as ANOLE but do not use Berlekamp's volume to construct preference queries.
\begin{itemize}
	\item \textbf{Greedy Binary Search.} We conduct a simple and direct implementation of binary-search-like task inference. When generating preference queries, it simply ignores all candidates that have made at least one wrong preference predictions, and the query trajectory pair only aims to partition the remaining candidate pool into two balanced separations.
	\item \textbf{Random Query.} We also include the simplest baseline that constructs the preference query by drawing a random trajectory pair from the experience buffer. This baseline serves an ablation study to investigate the benefits of removing the inference network. 
\end{itemize}

In addition, we implement two variants of PEARL \citep{rakelly2019efficient} that use a probabilistic context encoder to infer task specification from preference-based feedback.
\begin{itemize}
	\item \textbf{PEARL.} We modify PEARL's context encoder to handle the preference-based feedback. More specifically, the context encoder is a LSTM-based network that takes an ordered trajectory pair as inputs to make up the task embedding $z$, \ie, the input trajectory pair has been sorted by oracle preference. In meta-training, the preference is labeled by the ground-truth rewards from the environment simulator. During adaptation, the PEARL-based agent draws random trajectory pairs from the experience buffer to query the preference oracle, and the oracle feedback is used to update the posterior of task variable $z$.
	\item \textbf{PEARL+Augmentation.} We implement a data augmentation method for PEARL's meta-training to pursue the error tolerance of preference-based inference network. We impose random errors to the preference comparison when training the preference-based context encoder so that the inference network is expected to have some extent of error tolerance.
\end{itemize}

We include more implementation details of these baselines in Appendix~\ref{apx:pearl_implementation}.

\subsection{Performance Evaluation on MuJoCo-based Meta-RL Benchmark Tasks}
\label{sec:main-experiments}

\begin{figure}
	\centering
	\includegraphics[width=0.98\linewidth]{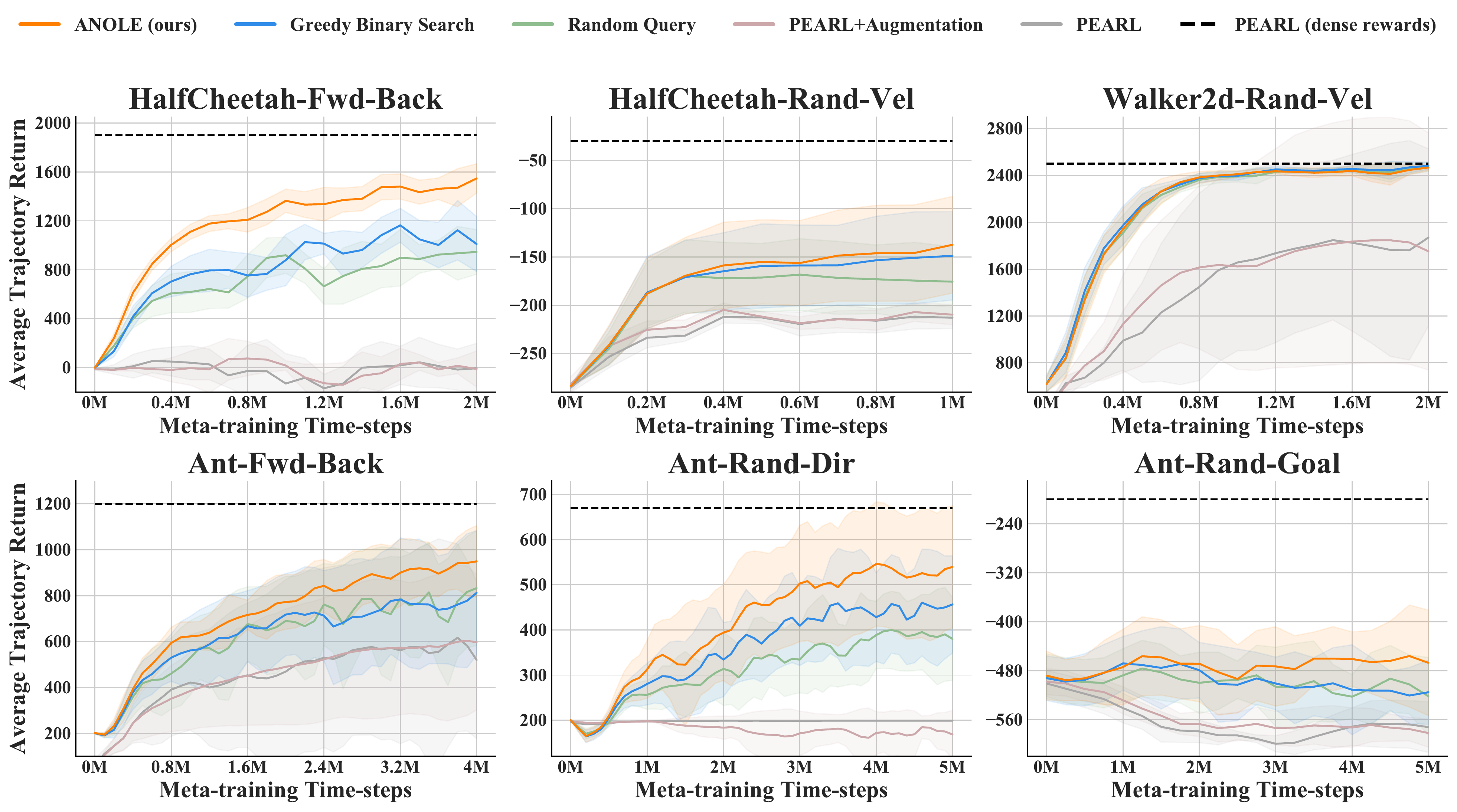}
	\caption{Learning curves on a suite of MuJoCo-based meta-RL benchmark tasks with preference-based adaptation. All curves plot the average performance from eight runs with random initialization. The shaded region indicates the standard deviation. ``PEARL (dense rewards)'' denotes the final performance of ordinary PEARL using dense reward signals for meta-testing.}
	\label{fig:main_results}
	\vspace{-0.1in}
\end{figure}

Figure~\ref{fig:main_results} presents the performance evaluation of ANOLE and baseline algorithms on a suite of meta-RL benchmark tasks with noisy preference oracle. The adaptation algorithms are restricted to use $K=10$ preference queries, and the noisy oracle would give wrong feedback for each query with probability $\epsilon=0.2$. We configure ANOLE with $K_E=2$ to compute Berlekamp's volume. The experiment results indicate two conclusions:
\begin{enumerate}
	\item Berlekamp's volume improves the error tolerance of preference-based task inference. The only difference between ANOLE and the first two baselines, \textit{Greedy Binary Search} and \textit{Random Query}, is the utilization of Berlekamp's volume, which leads to a significant performance gap on benchmark tasks.
	\item The non-parametric inference framework of ANOLE improves the scalability of preference-based few-shot adaptation. Note that, using ANOLE's algorithmic framework (see Algorithm~\ref{alg:framework}), a random query strategy can also outperform PEARL-based baselines. It may because the inference network used by PEARL cannot effectively translate the binary preference feedback to the task information.
\end{enumerate}

\subsection{Ablation Studies on the Magnitude of Oracle Noise} \label{sec:noise_ablation}

To demonstrate the functionality of Berlekamp's volume on improving error tolerance, we conduct an ablation study on the magnitude of oracle noise. In Figure~\ref{fig:noise_ablation}, we contrast the performance of ANOLE in adaptation settings with and without oracle noises, \ie, $\epsilon=0.2$ \textit{vs.} $\epsilon=0.0$. When the preference oracle does not carry any error, the simple greedy baseline can achieve the same performance as ANOLE. If we impose noises to the oracle feedback, ANOLE would suffer from a performance drop, where the drop magnitude is much more moderate than that of the greedy strategy. This result indicates that Berlekamp's volume does provide remarkable benefits on improving error tolerance whereas cannot completely eliminate the negative effects of oracle noises. It opens up a problem for future work to further improve the error tolerance of preference-based few-shot adaptation. In Appendix~\ref{apx:noise_ablation}, we conduct more ablation studies to understand the algorithmic functionality of ANOLE.

\begin{figure}
	\centering
	\includegraphics[width=0.98\linewidth]{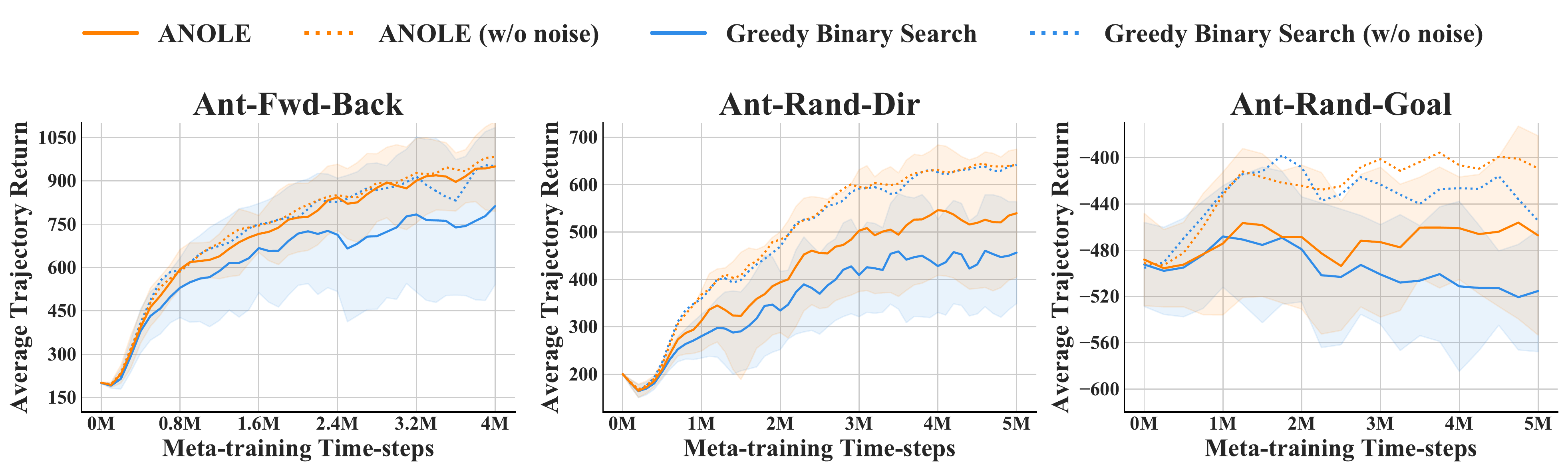}
	\caption{Investigating the impact of oracle noises on performance of preference-based adaptation. The dotted curves with tag "(w/o noise)" refer to the performance evaluation with a noiseless oracle.}
	\label{fig:noise_ablation}
	\vspace{-0.1in}
\end{figure}


\subsection{Experiments with Human Feedback}

\begin{wraptable}[9]{R}{0.44\linewidth}
	\centering\vspace{-0.17in}
	\caption{Evaluating the performance of ANOLE using human feedback.}
	\label{table:human-experiments}
	\begin{tabular}{cc}
		\toprule
		Task & ANOLE \\
		\midrule
		HalfCheetah-Fwd-Back & 1734.3$\pm$10.3 \\
		Ant-Fwd-Back & 931.1$\pm$38.0 \\
		Ant-Rand-Dir & 644.8$\pm$63.2 \\
		\bottomrule
	\end{tabular}
\end{wraptable}
We conduct experiments with real human feedback on MuJoCo-based meta-RL benchmark tasks. The meta-testing task specifications are generated by the same rule as the experiments in section~\ref{sec:main-experiments}. To facilitate human participation, we project the agent trajectory and the goal direction vector to a 2D coordinate system. The human participant watches the query trajectory pair and labels the preference according to the assigned task goal vector. The performance of ANOLE with human feedback is presented in Table~\ref{table:human-experiments}. Each entry is evaluated by 20 runs (\ie, using $20\times10=200$ human queries). We note that, in these experiments, the feedback accuracy of the human participant is better than the uniform-noise oracle considered in section~\ref{sec:main-experiments}. The average error rate of human feedback is $6.2\%$ over all experiments. The evaluation results are thus better than the experiments in Figure~\ref{fig:main_results}. We include a visualization of the human-agent interaction interface in Appendix~\ref{apx:human-interface}.
\section{Related Work}

\paragraph{Meta-Learning and Meta-RL.}
Modeling task specification by a latent task embedding (or latent task variable) is a widely-applied technique in meta-learning for both supervised learning \citep{rusu2019meta, gordon2019meta} and reinforcement learning \citep{rakelly2019efficient}. This paper studies the adaptation of latent task embedding based on preference-based supervision. One characteristic of our proposed approach is the non-parametric nature of the task inference protocol. Non-parametric adaptation has been studied in supervised meta-learning \citep{vinyals2016matching, snell2017prototypical, allen2019infinite} but is rarely applied to meta-RL. From this perspective, our algorithmic framework opens up a new methodology for few-shot policy adaptation.

\paragraph{Preference-based RL and Human-in-the-loop Learning.}
Learning from human's preference ranking is a classical problem formulation of human-in-the-loop RL \citep{akrour2011preference, akrour2012april, furnkranz2012preference, wilson2012bayesian}. When combining with deep RL techniques, recent advances focus on learning an auxiliary reward function that decomposes the trajectory-wise preference feedback to per-step supervision \citep{christiano2017deep, ibarz2018reward}. Several methods have been developed to generate informative queries \citep{lee2021pebble}, improve the efficiency of data utilization \citep{park2022surf}, and develop specialized exploration for preference-based RL \citep{liang2022reward}. In this paper, we open up a new problem for preference-based RL, \ie, preference-based few-shot policy adaptation. The potential applications of our problem setting may be similar to that of personalized adaptation \citep{yu2021personalized, wang2021preference}, a supervised meta-learning problem for modeling user preference. A future work is considering a wider range of human supervision, such as human attention \citep{zhang2020atari} and human annotation \citep{guan2021widening}.

\section{Conclusion and Discussions}

In this paper, we study the problem of few-shot policy adaptation with preference-based feedback and propose a novel meta-RL algorithm, called \textit{Adaptation with Noisy OracLE} (ANOLE). Our method leverages a classical problem formulation called \RenyiUlam's game to model the task inference problem with a noisy preference oracle. This connection to information theory enable us to extend the technique of Berlekamp's volume to establish an error-tolerant approach for preference-based task inference, which is demonstrated as a promising approach on an extensive set of benchmark tasks.

We conclude this paper by discussing limitations, future works, and other relevant aspects that have not been covered.

\paragraph{Adaptation to Environment Dynamics.}
One limitation of ANOLE is that it does not consider the potential shift of transition dynamics from meta-training environments to meta-testing environments. Our problem formulation assumes only the reward function alters in the adaptation phase (see section~\ref{sec:problem-formulation}). The current implementation of ANOLE cannot handle the adaptation to the changing environment dynamics. One promising way to address this issue is to integrate ANOLE with classical meta-RL modules to model the probabilistic inference regarding the transition function. In addition, it is critical to investigate the behavior of human preference when the query trajectories may contain unrealistic transitions.

\paragraph{Expressiveness of Preference Partial Ordering.}
A recent theoretical work indicates that the trajectory-wise partial-order preference can define richer agent behaviors than step-wise reward functions \citep{abel2021expressivity}. However, this superior of preference-based learning has rarely been shown in the empirical studies. \eg, in our experiments, the preference feedback is simulated by summing step-wise rewards given by the MuJoCo simulator, which is a common setting used by most preference-based RL works. In addition, most advanced preference-based RL algorithms are built on a methodology that linearly decomposes the trajectory-wise preference supervision to a step-wise auxiliary reward function, which degrades the expressiveness of the partial ordering preference system. These problems are fundamental challenges to preference-based learning.

\begin{ack}
	This work is supported by National Key R\&D Program of China No. 2021YFF1201600.
\end{ack}

\bibliography{ref}
\bibliographystyle{plainnat}

\clearpage

\appendix

\section{Omitted Proofs} \label{apx:proof}

The proofs of these propositions are extended from \citet{berlekamp1968block}.

\VolumeProposition*

\begin{proof}
	Note that both oracle's preference feedback and $z$-conditioned preference prediction are binary values. \ie, for a given preference query $(\tau_{k}^{(1)}, \tau_{k}^{(2)})$, the prediction given by each task embedding $z$ is either correct or wrong. The mismatch count must be updated to one of the following cases:
	\begin{itemize}
		\item Case \texttt{\#}1: $\mismatch(z;\mathcal{Q}^{(1)\succ(2)}_{k})=\mismatch(z;\mathcal{Q}_{k-1})$ and $\mismatch(z;\mathcal{Q}^{(2)\succ(1)}_{k})=\mismatch(z;\mathcal{Q}_{k-1})+1$;
		\item Case \texttt{\#}2: $\mismatch(z;\mathcal{Q}^{(2)\succ(1)}_{k})=\mismatch(z;\mathcal{Q}_{k-1})$ and $\mismatch(z;\mathcal{Q}^{(1)\succ(2)}_{k})=\mismatch(z;\mathcal{Q}_{k-1})+1$.
	\end{itemize}

	It implies that
	\begin{align*}
		vol_z(\mathcal{Q}_{k-1}) &= \sum_{\ell=0}^{K_{E}-\mismatch(z;\mathcal{Q}_{k-1})} \binom{K-(k-1)}{\ell} \\
		&= \sum_{\ell=0}^{K_{E}-\mismatch(z;\mathcal{Q}_{k-1})} \left(\binom{K-k}{\ell} + \binom{K-k}{\ell-1}\right) \\
		&= \sum_{\ell=0}^{K_{E}-\mismatch(z;\mathcal{Q}_{k-1})} \binom{K-k}{\ell} + \sum_{\ell=0}^{K_{E}-(\mismatch(z;\mathcal{Q}_{k-1})+1)} \binom{K-k}{\ell} \\
		&= \sum_{\ell=0}^{K_{E}-\mismatch(z;\mathcal{Q}^{(1)\succ(2)}_{k})} \binom{K-k}{\ell} + \sum_{\ell=0}^{K_{E}-\mismatch(z;\mathcal{Q}^{(2)\succ(1)}_{k})} \binom{K-k}{\ell} \\
		&= vol_z(\mathcal{Q}^{(1)\succ(2)}_{k}) + vol_z(\mathcal{Q}^{(2)\succ(1)}_{k}).
	\end{align*}

	By plugging into Eq.~\eqref{eq:volume_def}, we have
	\begin{align*}
		\BVol(\mathcal{Q}_{k-1}) &= \sum_{z\in\widehat Z} vol_z(\mathcal{Q}_{k-1}) \\
		&= \sum_{z\in\widehat Z} \Bigl( vol_z(\mathcal{Q}^{(1)\succ(2)}_{k}) + vol_z(\mathcal{Q}^{(2)\succ(1)}_{k}) \Bigr) \\
		&= \BVol(\mathcal{Q}^{(1)\succ(2)}_{k}) + \BVol(\mathcal{Q}^{(2)\succ(1)}_{k}).
	\end{align*}
\end{proof}

\UnitProposition*

\begin{proof}
	Note that the values of $\BVol(\cdot)$ and $vol_z(\cdot)$ are non-negative integers. $\BVol(\mathcal{Q}_k)=1$ implies there exists exactly one task embedding $z\in\widehat Z$ satisfying $vol_z(\mathcal{Q}_k)=1=\binom{0}{0}$, which further implies $\mismatch(z;\mathcal{Q}_k)=K_{E}$.
\end{proof}

\clearpage
\section{Experiment Setting and Implementation Details}

\subsection{Experiment Setting}
\label{apx:experiment_setting}

We adopt the environment setting created by \citet{rothfuss2019promp}. This benchmark is a suite of MuJoCo locomotion tasks, where the reward function are varied to create a multi-task setting. More specifically, there are four ways to vary the task specification $\mathcal{T}$:
\begin{itemize}
	\item \texttt{Fwd-Back:} The task variable $\mathcal{T}$ varies the target direction within \{forward/backward\};
	\item \texttt{Rand-Vel:} The task variable $\mathcal{T}$ varies the target velocity within a bounded range;
	\item \texttt{Rand-Dir:} The task variable $\mathcal{T}$ varies the target direction within the 2D-plane;
	\item \texttt{Rand-Goal:} The task variable $\mathcal{T}$ varies the target location within a bounded area.
\end{itemize}
The training and testing tasks are randomly generated by a fixed random seed. \ie, the generation of training/testing tasks do not vary across runs. During meta-training, the meta-RL algorithm has the full access to the environmental interaction. The algorithms can obtain trajectories with both transition and reward information. During meta-testing, the reward function would become unavailable to the meta-RL agent. The agent can only query a preference oracle to extract information about the task specification $\mathcal{T}$. The preference oracle $\Omega_{\mathcal{T}}$ is simulated by comparing the ground-truth trajectory return given by the MuJoCo simulator, \ie, the oracle can access to the ground-truth reward function. We consider this asymmetric supervision setting since this paper only focus on the design of adaptation protocol, and our proposed adaptation algorithm can be plugged in any meta-training algorithm using latent embeddings.

In our experiments, all networks are trained using a single GPU and a single CPU core.
\begin{itemize}
	\item GPU: GeForce GTX 1080 Ti;
	\item CPU: Intel(R) Xeon(R) CPU E5-2630 v4 @ 2.20GHz.
\end{itemize}
In each run of experiment, 4M steps of training can be completed within 24 hours.

\subsection{Implementation Details of ANOLE} \label{apx:our_implementation}

The overall meta-training procedure of ANOLE is implemented upon PEARL \citep{rakelly2019efficient} with some modifications and incremental designs.

\paragraph{Probabilistic Embedding.}
Note that, different from most PEARL-based algorithms, ANOLE does not include an inference network. Instead, we use a set of trainable variables to model the latent task embedding of each training task. To expand the latent space and promote generalization, we assign each training task $\mathcal{T}_i$ a multivariate Gaussian $\mathcal{N}(\mu_i,\sigma_i)$ with zero covariance, where $\mu_i,\sigma_i^2\in\mathbb{R}^d$ and $d$ denotes the dimension of latent space. More specifically, we set up $2dN$ trainable variables instead of an inference network to model the latent space from training tasks $\{\mathcal{T}_i\}_{i=1}^N$. The same as PEARL, we conduct a regularization loss to make the learned latent space compact:
\begin{align*}
	\mathcal{L}^{\text{KL}}(\mu,\sigma) &= \frac{1}{N}\sum_{i=1}^N D_{\text{KL}}(\mathcal{N}(\mu_i,\sigma_i^2)~\|~\mathcal{N}(\mathbf{0}_d,\mathbf{1}_d)),
\end{align*}
where $\mathcal{N}(\mathbf{0}_d,\mathbf{1}_d)$ denotes a $d$-dimensional standard multivariate Gaussian.

\paragraph{Policy Training.}
We adopt the same off-policy meta-RL framework as PEARL to train the policy. We extend soft actor-critic \citep[SAC;][]{haarnoja2018soft} with the latent task embedding $z$, \ie, the policy and value functions are represented by $\pi_\theta(a\mid s;z)$ and $Q_\phi(s,a;z)$ where $\theta$ and $\phi$ denote the parameters. The objective function for actor and critic networks are presented as below:
\begin{align*}
	\mathcal{L}^{\text{actor}}(\theta) &= \mathop{\mathbb{E}}_{(\mathcal{T}_i,s)\sim\mathcal{D}}\left[ \left. D_\text{KL}\left(\pi_\theta(\cdot|s)\left\|\frac{\exp(Q_\phi(s, \cdot))}{\int\exp(Q_\phi(s, a))\mathrm{d}a}\right.\right) \right| z\sim \mathcal{N}(\mu_i,\sigma_i^2) \right], \\
	\mathcal{L}^{\text{critic}}(\phi) &= \mathop{\mathbb{E}}_{(\mathcal{T}_i,s,a,r,s')\sim\mathcal{D}}\left[ \left.\bigl(Q_\phi(s,a;z)-r-\gamma V_{\phi_{\text{target}}}(s';z)\bigr)^2 \right| z\sim \mathcal{N}(\mu_i,\sigma_i^2) \right].
\end{align*}
This part of implementation, including network architecture and optimizers, is reused from the open-source code of PEARL.

\paragraph{Preference Predictor.}
We train a $z$-conditioned preference predictor $f_\psi(\cdot;z)$ on the meta-training tasks by optimizing the preference loss function:
\begin{align*}
	\mathcal{L}^{\text{Pref}}(\psi) = \mathop{\mathbb{E}}\left[ D_{\text{KL}}\left(\left.\mathbb{I}\bigl[\tau^{(1)}\succ\tau^{(2)}\mid\mathcal{T}\bigr] ~\right\|~ f_\psi\bigl(\tau^{(1)}\succ\tau^{(2)};z\bigr)\right) \right],
\end{align*}
where $\psi$ denotes the parameters of the preference predictor, $D_{\text{KL}}(\cdot\|\cdot)$ denotes the Kullback–Leibler divergence, and $\mathbb{I}[\tau^{(1)}\succ\tau^{(2)}\mid\mathcal{T}]$ is the ground-truth preference order specified by the task specification $\mathcal{T}$ (\eg, specified by the reward function on training tasks). The trajectory pair $\langle\tau^{(1)},\tau^{(2)}\rangle$ is drawn from the experience buffer, and $z$ is the task embedding vector encoding $\mathcal{T}$. Optimizing this KL-based loss function is equivalent to optimizing binary cross entropy.

Following the implementation of preference-based RL \citep{christiano2017deep, lee2021pebble}, we use Bradley-Terry model \citep{bradley1952rank} to establish a preference predictor:
\begin{align} \label{eq:preference-predictor}
	f_{\psi}(\tau^{(1)}\succ\tau^{(2)}; z) &= \frac{\exp\left(\sum_t g_{\psi}(s^{(1)}_t, a^{(1)}_t; z)\right)}{\exp\left(\sum_t g_{\psi}(s^{(1)}_t, a^{(1)}_t; z)\right) + \exp\left(\sum_t g_{\psi}(s^{(2)}_t, a^{(2)}_t; z)\right)},
\end{align}
where $g_\psi(s,a;z)$ is a network that outputs the ranking score of state-action pair $(s,a)$. In implementation, $\tau^{(1)}$ and $\tau^{(2)}$ refer to two fixed-length trajectory segments instead of considering the complete trajectory. A future work is adopting the random-sampling trick \citep{ren2022learning} for the extension to long-horizon preference.

\paragraph{Batch-Constrained Embedding Sampling.}
A pre-processing step of our task inference algorithm is sampling a candidate pool of task embeddings (see line~\ref{alg:sample_pool} in Algorithm~\ref{alg:framework}) for the subsequent embedding selection. We restrict the support of this candidate pool to the task embedding distribution conducted in meta-training. We call this procedure batch-constrained embedding sampling, since it corresponds to the notion of batch-constrained policy \citep{fujimoto2019off} in the literature of offline reinforcement learning. This restriction ensures the induced policy $\pi_\theta(a|s;z)$ is covered by the training distribution, so that the meta-testing policies would not suffer from unpredictable out-of-distribution generalization errors. More specifically, we sample a set of task embeddings from the mixture distribution of training tasks, $\widehat Z\gets\{z_j \sim q(z)\}_{j=1}^M$ where $q(z)$ refers to the mixture of training task variable:
\begin{align*}
	q(z) = \frac{1}{N}\sum_{i=1}^N \mathcal{N}(z \mid \mu_i,\sigma_i^2).
\end{align*}

\paragraph{Candidate Pool Size.}
Note that the number of embedding candidates initialized in $\widehat Z$ may affect the computation of Berlekamp's volume. The size of candidate pool $\widehat Z$, denoted by $M$, is determined by the following formula:
\begin{align*}
	M = \left\lfloor \frac{2^K}{vol_z(\emptyset)} \right\rfloor = \left\lfloor \frac{2^K}{\sum_{\ell=0}^{K_E}\binom{K}{\ell}} \right\rfloor,
\end{align*}
where $z$ is an arbitrary embedding candidate. This configuration of candidate pool size ensures that, with ideal preference queries, Berlekamp's volume $\BVol(\mathcal{Q}_K)$  can be reduced to 1 by $K$ queries. \ie, in the ideal situation, each preference query can halve the value of $\BVol(\mathcal{Q}_k)$. Note that, our algorithm does not require $\BVol(\mathcal{Q}_K)$ to be reduced to 1, since finding the task embedding with minimum mismatch is always a plausible solution (see $\arg\min_{z\in\widehat Z} \mismatch(z;\mathcal{Q}_K)$ at line~\ref{alg:return} of Algorithm~\ref{alg:framework}).

\paragraph{Error-Tolerant Query Strategy.}
In our implementation, we use mini-batch sampling to approximate the query generation protocol $\textsc{GenerateQuery}(\cdot)$ in Eq.~\eqref{eq:query_generation}. We sample a mini-batch of trajectory pairs and find the best trajectory pair within the sampled mini-batch.
\begin{align*}
	\textsc{GenerateQuery}(\widehat Z, \mathcal{Q}_{k-1}) = \mathop{\arg\min}_{(\tau_k^{(1)}, \tau_k^{(2)})\in B} \biggl(\max\Bigl\{ \BVol(\mathcal{Q}^{(1)\succ(2)}_{k}), ~\BVol(\mathcal{Q}^{(2)\succ(1)}_{k}) \Bigr\}\biggr),
\end{align*}
where $B$ denotes a mini-batch of trajectory pairs that are uniformly sampled from the experience replay buffer. In our implementation, we sample 100 trajectory pairs for each mini-batch $B$.

\clearpage

\paragraph{Hyper-Parameters.}
We summarize major hyper-parameters as the following table. We use this set of hyper-parameters for all ANOLE's experiments.

\begin{table}[H]
	\centering
	\begin{tabular}{c c}
		\toprule
		Hyper-Parameter & Default Configuration \\
		\midrule
		dimension of latent embedding $d$ & 5 \\
		discount factor $\gamma$ & 0.99 \\
		\midrule
		optimizer (all losses) & Adam \citep{kingma2015adam} \\
		learning rate & $3\cdot 10^{-4}$ \\
		Adam-$(\beta_1, \beta_2, \epsilon)$ & $(0.9, 0.999, 10^{-8})$ \\
		\midrule
		temperature $\alpha$ & 1.0 \\
		Polyak-averaging coefficient & 0.005 \\
		\texttt{\#} gradient steps per environment step & 1/5 \\
		\texttt{\#} gradient steps per target update & 1 \\
		\midrule
		\texttt{\#} transitions in replay buffer (for each task $\mathcal{T}$) & $10^6$ \\
		\texttt{\#} tasks in each mini-batch for training SAC & 10 \\
		\texttt{\#} transitions in each task-batch for training SAC & 256 \\
		\texttt{\#} trajectory segments in each mini-batch for training $f_\psi$ & 10 \\
		\texttt{\#} transitions in each trajectory segment for training $f_\psi$ & 64 \\
		\midrule
		\texttt{\#} preference queries $K$ & 10 \\
		\texttt{\#} wrong feedbacks to tolerate $K_E$ & 2 \\
		\texttt{\#} trajectory pairs to approximate $\textsc{GenerateQuery}$ & 100 \\
		\bottomrule
	\end{tabular}
\end{table}

\subsection{Implementation Details of PEARL-based Baselines} \label{apx:pearl_implementation}

We slightly modify the implementation of PEARL to make it work for preference-based adaptation.

\paragraph{Preference-based Context Encoder.}
We modify PEARL's context encoder to handle the preference-based feedback. We use a LSTM-based context encoder that takes an ordered trajectory pair as inputs to make up the task embedding $z$, \ie, the input trajectory pair has been sorted by oracle preference. The architecture of LSTM-based encoder is implemented by the open-source code of PEARL. The same as the original version of PEARL, the output of context encoder is a $d$-dimensional Gaussian. In meta-training, the preference is labeled by the ground-truth rewards from the environment simulator, \ie, comparing the sum of ground-truth rewards. During adaptation, the PEARL-based agent draws random trajectory pairs from the experience buffer to query the preference oracle, and the oracle feedback is used to update the posterior of task variable $z$. The posterior update rule is reused from the open source code of PEARL.

\paragraph{Data Augmentation for Error-Tolerance.}
We implement a data augmentation method for PEARL's meta-training to pursue the error tolerance of preference-based inference network. To mimic the error mode of noisy preference oracle, we impose random errors to the preference comparison when training the preference-based context encoder. The error is uniformly flipped preference feedback with probability 0.2. In this way, the preference-based context encoder is trained using the same noisy preference signals as the meta-testing procedure so that the inference module is expected to have some extent of error tolerance.

\paragraph{Hyper-Parameters.}
We do not modify any hyper-parameters of PEARL. Note that the open-source implementation of PEARL specializes different hyper-parameter configurations for different environments. In this paper, we conduct two environments, \texttt{Walker2d-Rand-Vel} and \texttt{And-Rand-Dir}, without official hyper-parameter configuration since PEARL does not evaluate on them. To address this issue, we transfer hyper-parameter configurations from similar tasks. For \texttt{Walker2d-Rand-Vel}, we use the same configuration as \texttt{HalfCheetah-Rand-Vel}. For \texttt{And-Rand-Dir}, we use the same configuration as \texttt{And-Rand-Goal}.

\clearpage

\section{Ablation Studies on the Magnitude of Oracle Noise} \label{apx:noise_ablation}

We evaluate the performance of ANOLE and baselines under different magnitudes of oracle noises $\epsilon$. These results are generated by a same set of runs. \ie, each run of meta-training evaluates several meta-testing configurations. When the oracle carries no error, the performance of ANOLE and greedy binary search are almost the same. With the noise magnitude increasing, the gap between ANOLE and baselines become larger.

\subsection{Performance Evaluation without Noises ($\epsilon=0.0$)}

\begin{figure}[H]
	\includegraphics[width=0.98\linewidth]{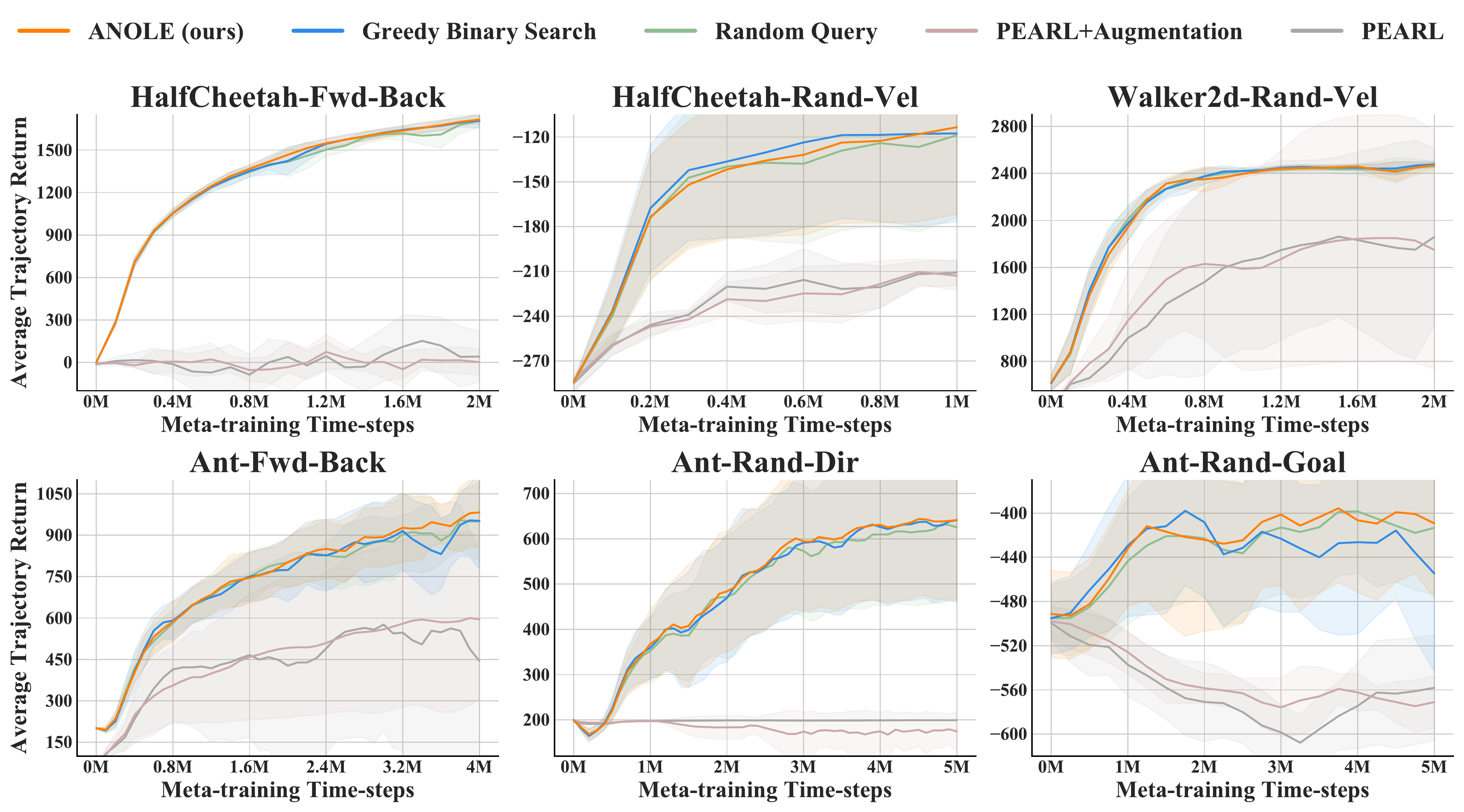}
\end{figure}

\subsection{Performance Evaluation with Noise Magnitude $\epsilon=0.1$}

\begin{figure}[H]
	\includegraphics[width=0.98\linewidth]{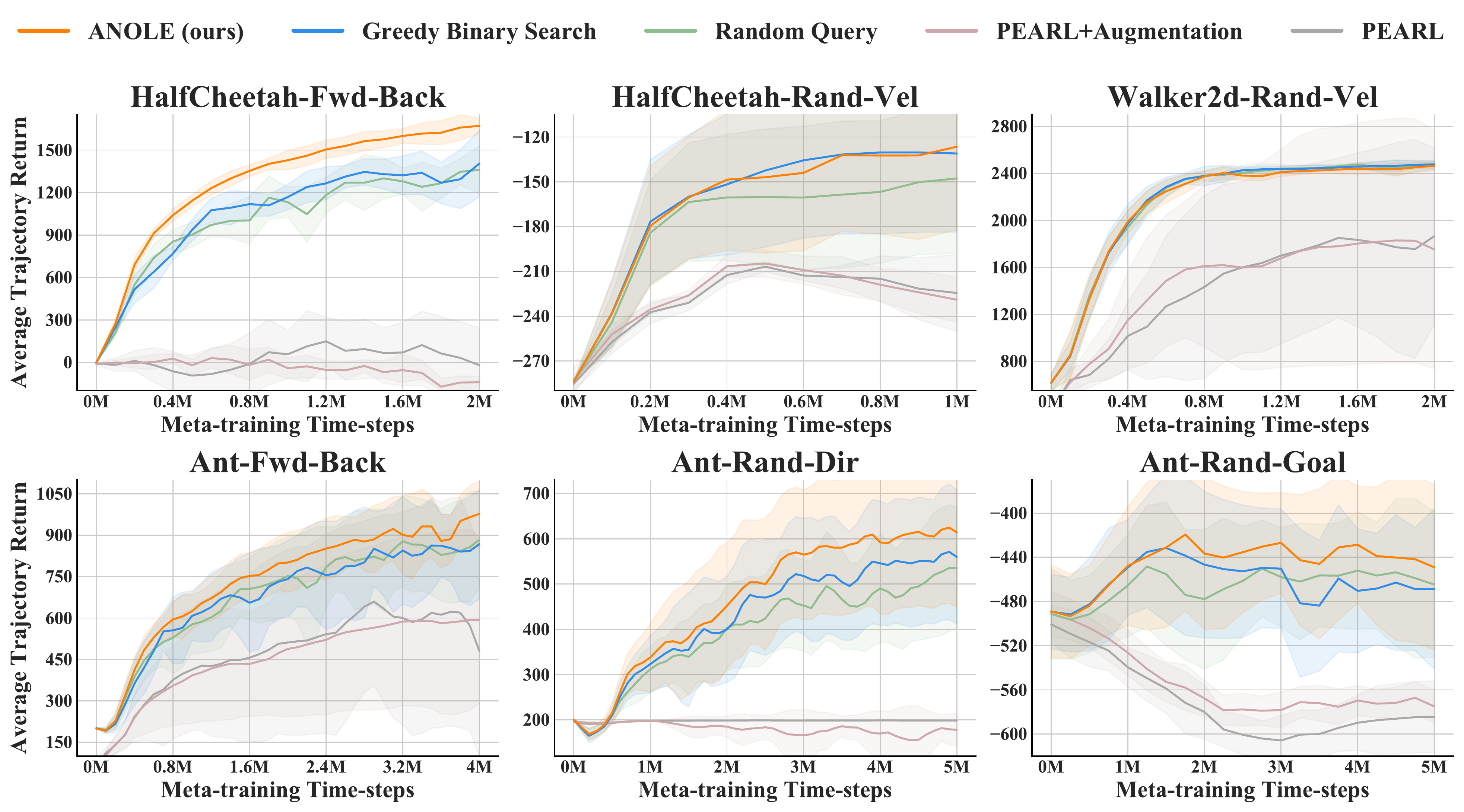}
\end{figure}

\clearpage

\subsection{Performance Evaluation with Noise Magnitude $\epsilon=0.2$ (Default Setting)}

\begin{figure}[H]
	\includegraphics[width=0.98\linewidth]{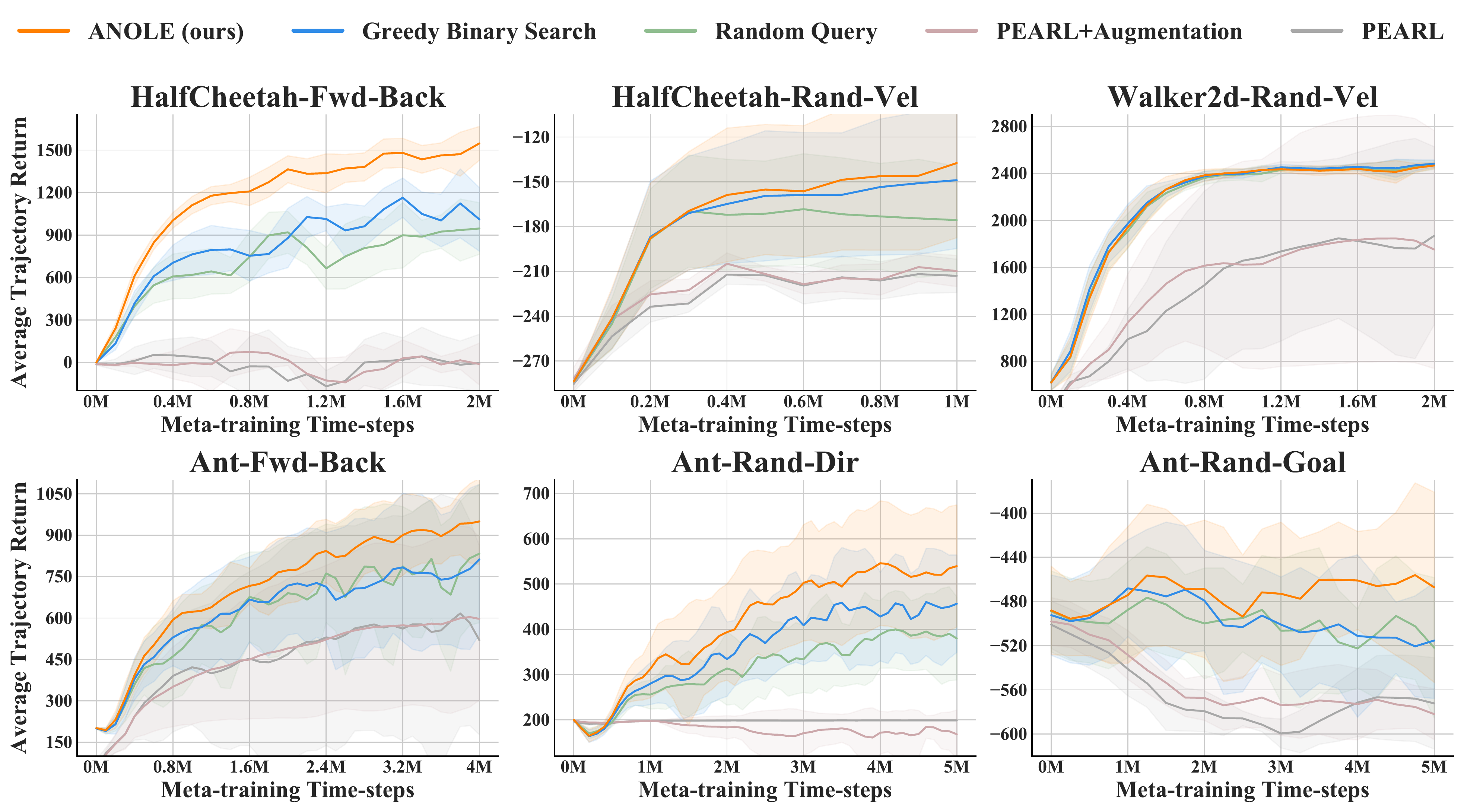}
\end{figure}

\subsection{Performance Evaluation with Noise Magnitude $\epsilon=0.3$}

\begin{figure}[H]
	\includegraphics[width=0.98\linewidth]{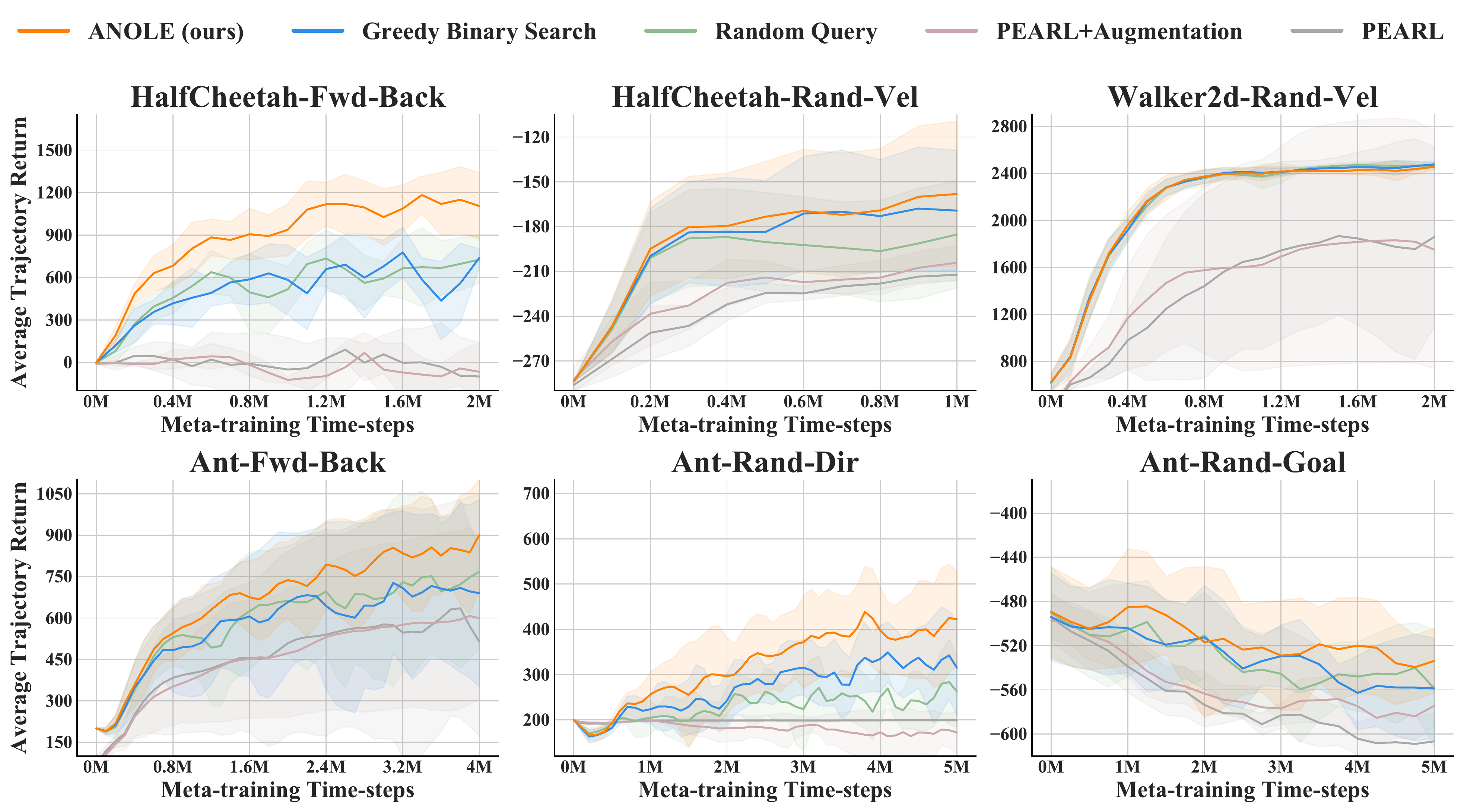}
\end{figure}

\clearpage

\section{Performance Evaluation at Each Adaptation Step}

In addition to the full learning curves, we plot the performance of final policy in every adaptation steps. Final policy refers to the last evaluation point presented in Figure~\ref{fig:main_results}. The experiments show that PEARL-based baselines cannot effectively extract task information from preference-based binary feedback.

\subsection{Preference-based Few-shot Adaptation with Noise Magnitude $\epsilon=0.2$ (Default Setting)}

\begin{figure}[H]
	\includegraphics[width=0.98\linewidth]{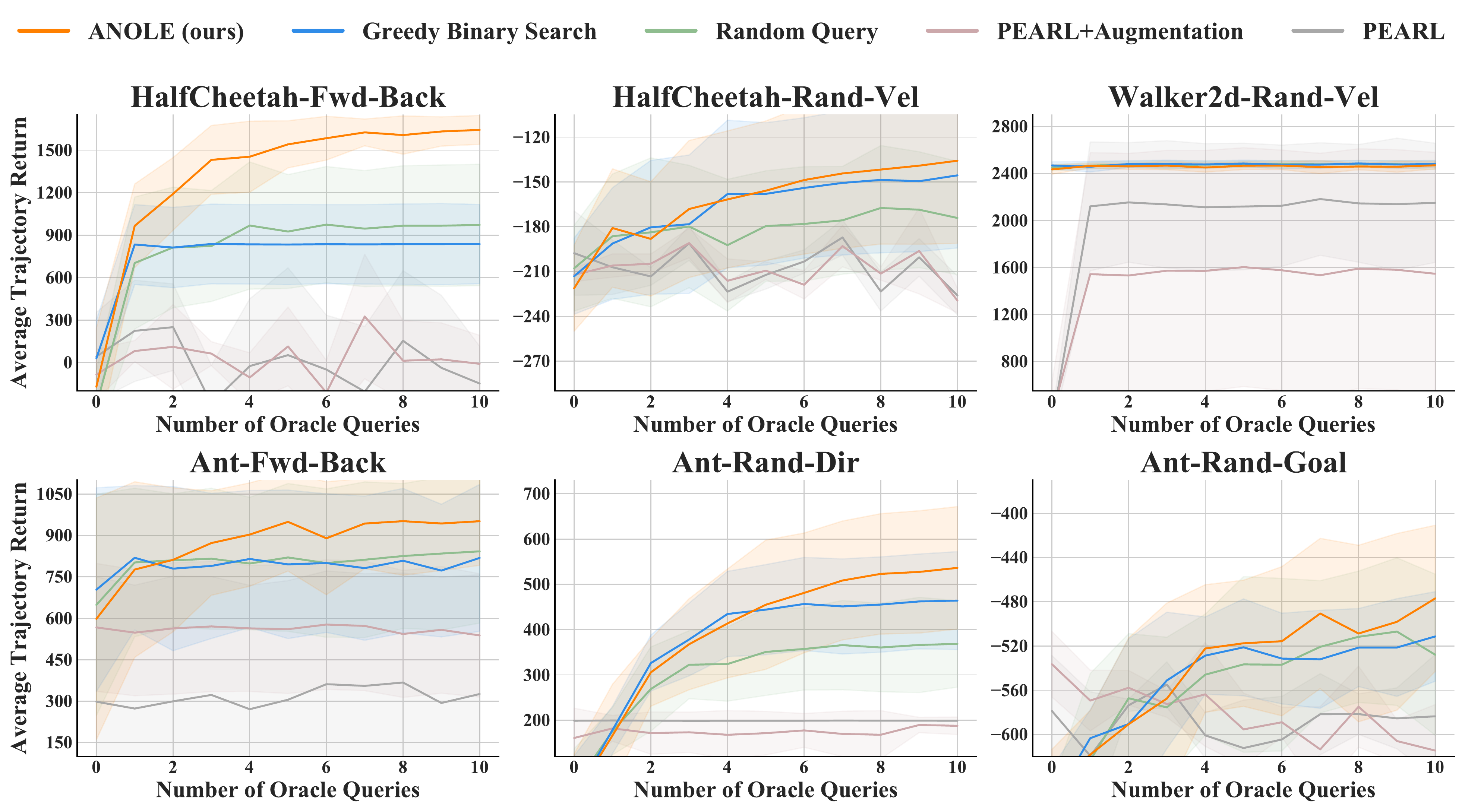}
\end{figure}

\subsection{Preference-based Few-shot Adaptation without Noises ($\epsilon=0.0$)}

\begin{figure}[H]
	\includegraphics[width=0.98\linewidth]{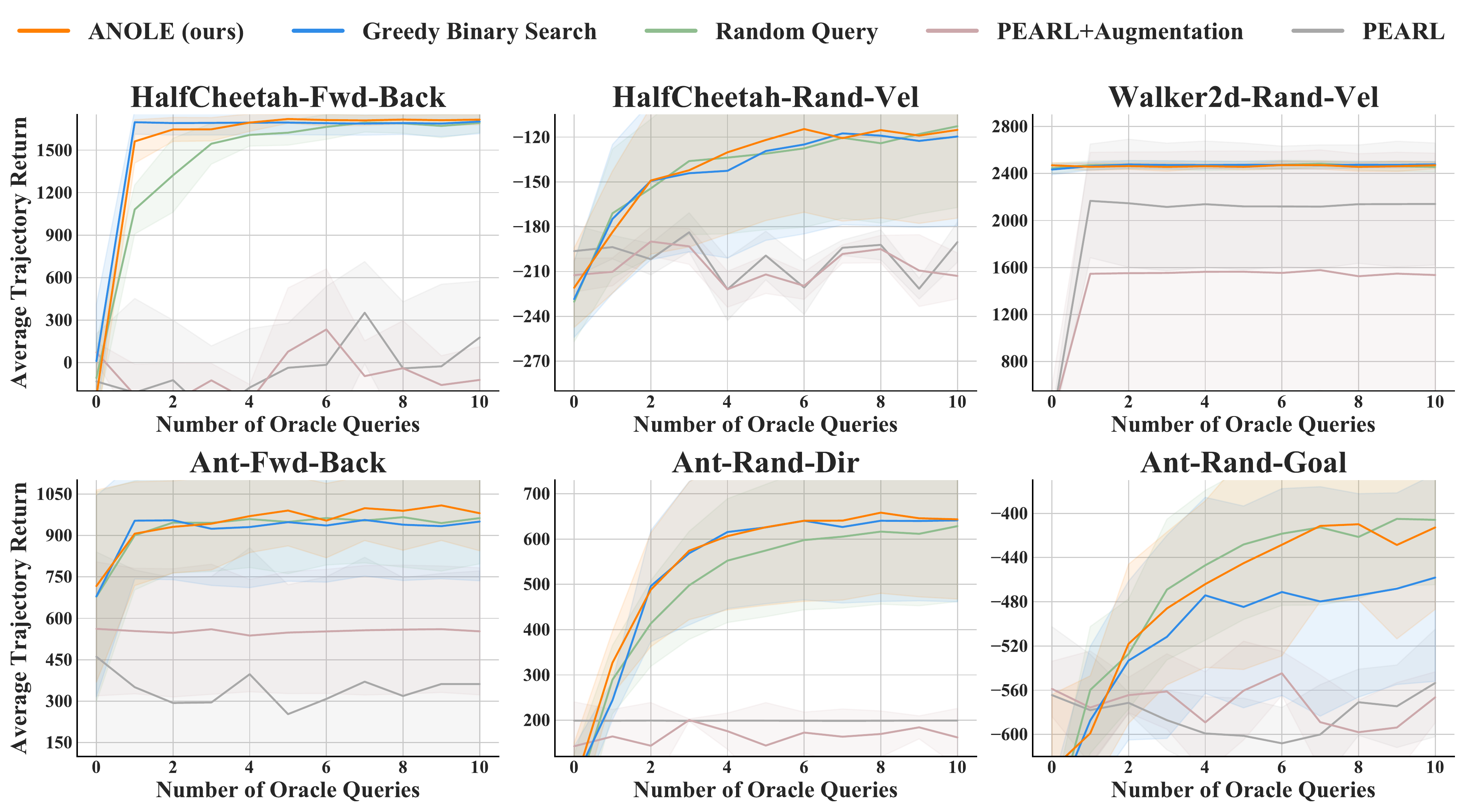}
\end{figure}

\clearpage

\section{ANOLE with An Alternative Meta-Training Module}

We investigate the performance of ANOLE with an alternative meta-training module. We use preference-based supervision to train the meta-policy. Following a classical paradigm of preference-based RL \citep{christiano2017deep}, we conduct a reward learning module to decompose the preference-based binary feedback into per-step reward supervision. More specifically, we use the ranking score $g_\psi(s,a;z)$ defined in Eq.~\eqref{eq:preference-predictor} as an auxiliary reward function, which is learned from preference comparisons. The experiments show that ANOLE with preference-based meta-training can significantly outperform PEARL-based baselines using environmental rewards.

\subsection{Performance Evaluation with Noise Magnitude $\epsilon=0.2$ (Default Setting)}

\begin{figure}[H]
	\includegraphics[width=0.98\linewidth]{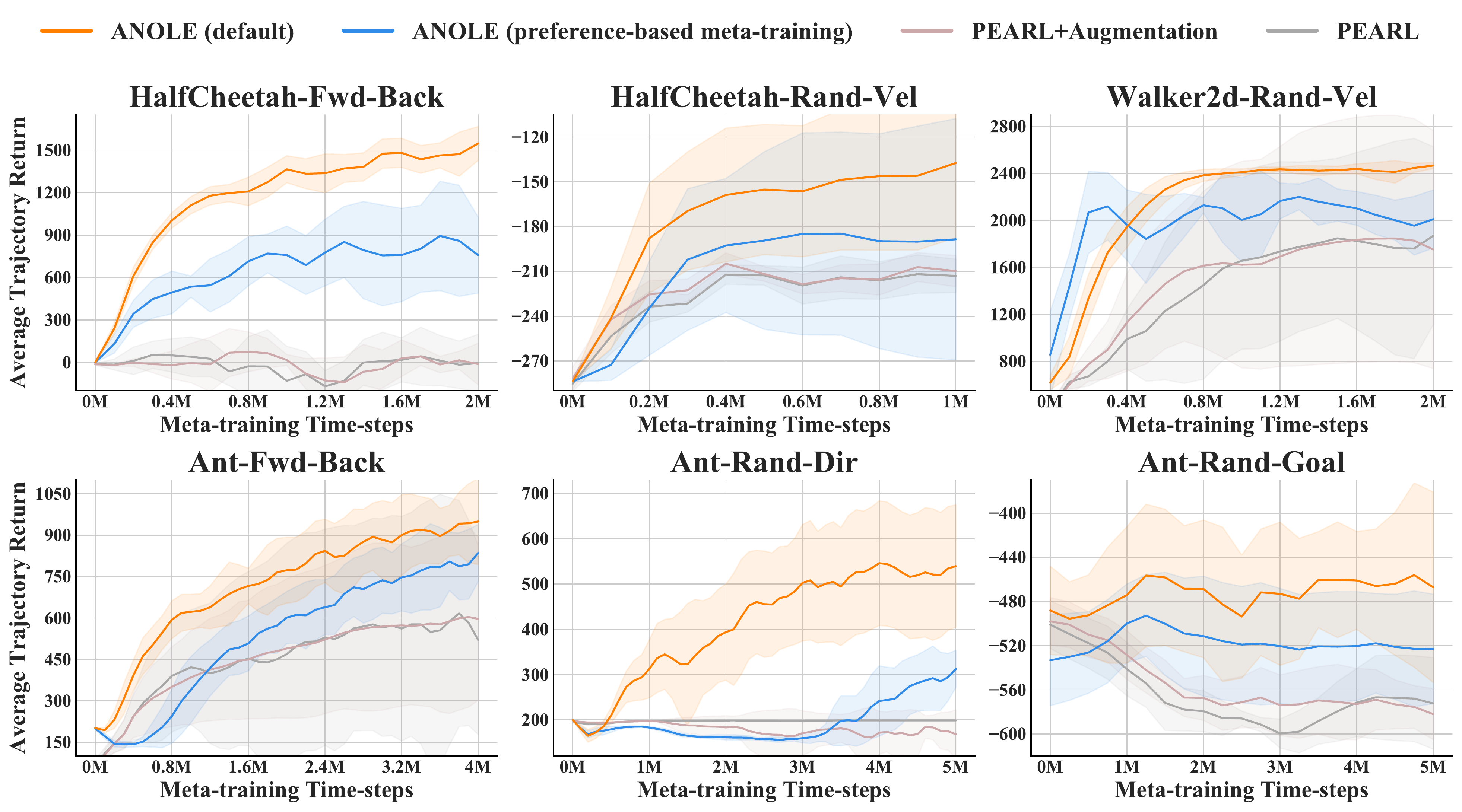}
\end{figure}

\subsection{Performance Evaluation without Noises ($\epsilon=0.0$)}

\begin{figure}[H]
	\includegraphics[width=0.98\linewidth]{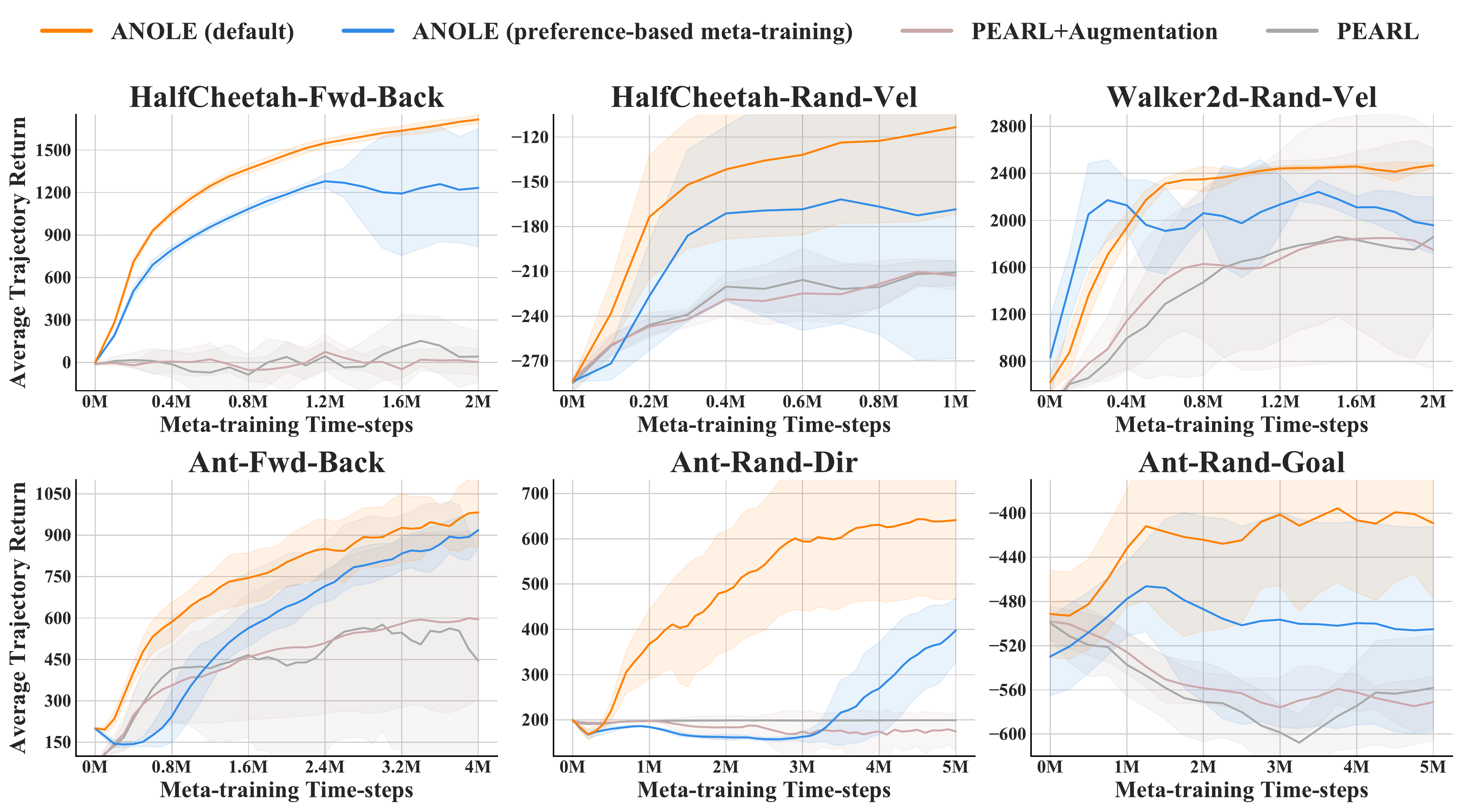}
\end{figure}

\clearpage
\section{Performance Evaluation under Two Alternative Noise Modes}

In addition to the uniform noise mode considered in the main paper, we conduct experiments with two additional noise modes:
\begin{itemize}
	\item \textbf{Boltzman Mode.} The oracle answers $\tau^{(1)}\succ\tau^{(2)}$ with the probability
	\begin{align*}
		\frac{\exp\left(\beta\cdot\text{Return}(\tau^{(1)})\right)}{\exp\left(\beta\cdot\text{Return}(\tau^{(1)})\right) + \exp\left(\beta\cdot\text{Return}(\tau^{(2)})\right)} = \frac{\exp\left(\beta\sum_t r^{(1)}_t\right)}{\exp\left(\beta\sum_t r^{(1)}_t\right) + \exp\left(\beta\sum_t r^{(2)}_t\right)}
	\end{align*}
	where $\beta$ denotes the temperature parameter. This error mode is commonly considered by recent preference-based RL works \citep{lee2021b}.
	\item \textbf{Hack Mode.} The oracle always gives wrong feedbacks for the first 20\% queries and keeps correct for the remaining 80\%. This noise mode is designed to hack search-based query strategies, since the first few queries are usually the most informative.
\end{itemize}

The experiments results are presented as follows. The learning curves tagged with label $(\beta=\cdot)$ correspond to the experiments with Boltzman noise mode. The learning curves in the last column correspond to the experiments with ``Hack'' noise mode.
\begin{figure}[H]
	\includegraphics[width=0.98\linewidth]{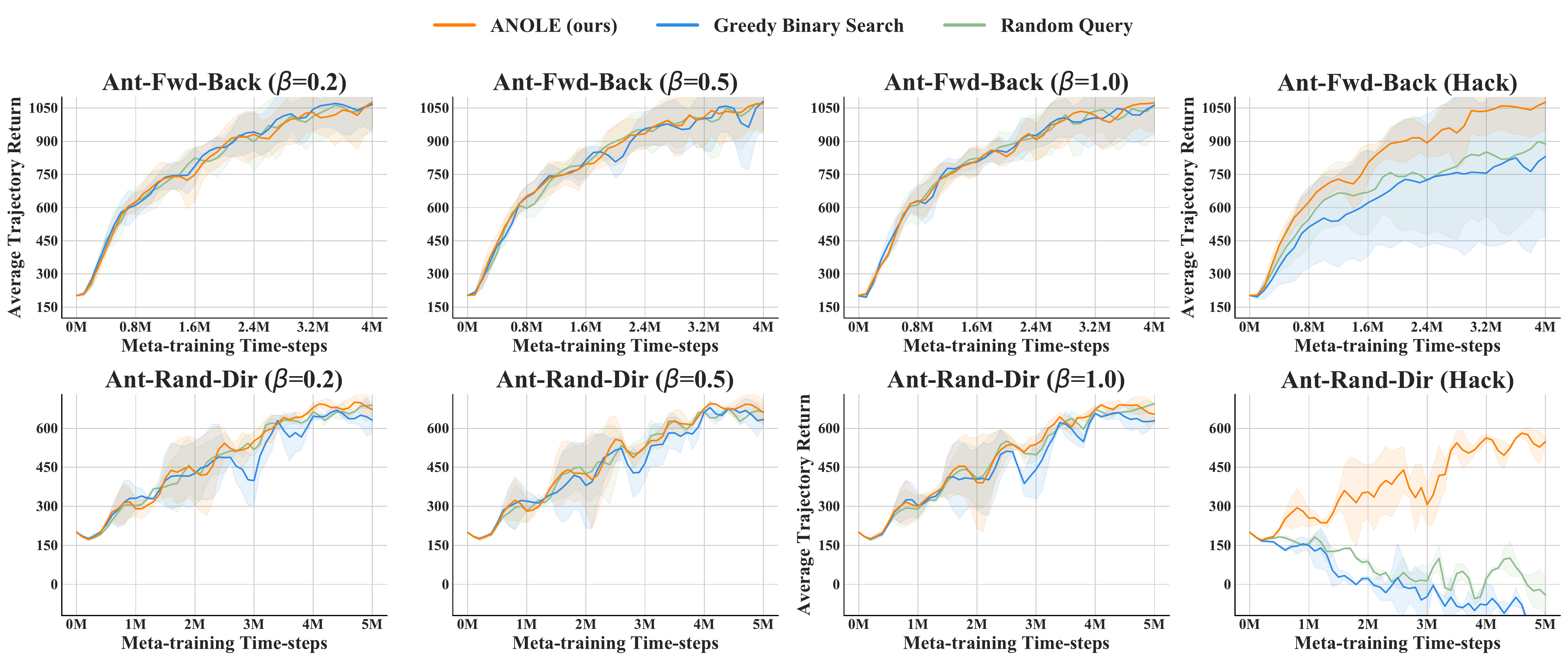}
\end{figure}

We note that, for these testing environments, Boltzman feedbacks are much more accurate than the uniform-noise oracle considered in the paper. That is because the generated query trajectory pairs can usually be clearly compared. The performance of ANOLE and greedy/random query strategies are comparable since errors hardly occur. Regarding this result, we would like to emphasize that the main purpose of ANOLE is to tolerate unintended errors of human users. More specifically, a robust algorithm is expected to tolerate a few amounts of irrational errors. 

\clearpage
\section{Interface of Human-ANOLE Interaction}
\label{apx:human-interface}

\begin{figure}[H]
	\begin{subfigure}[b]{0.32\linewidth}
		\centering
		\includegraphics[width=0.98\linewidth]{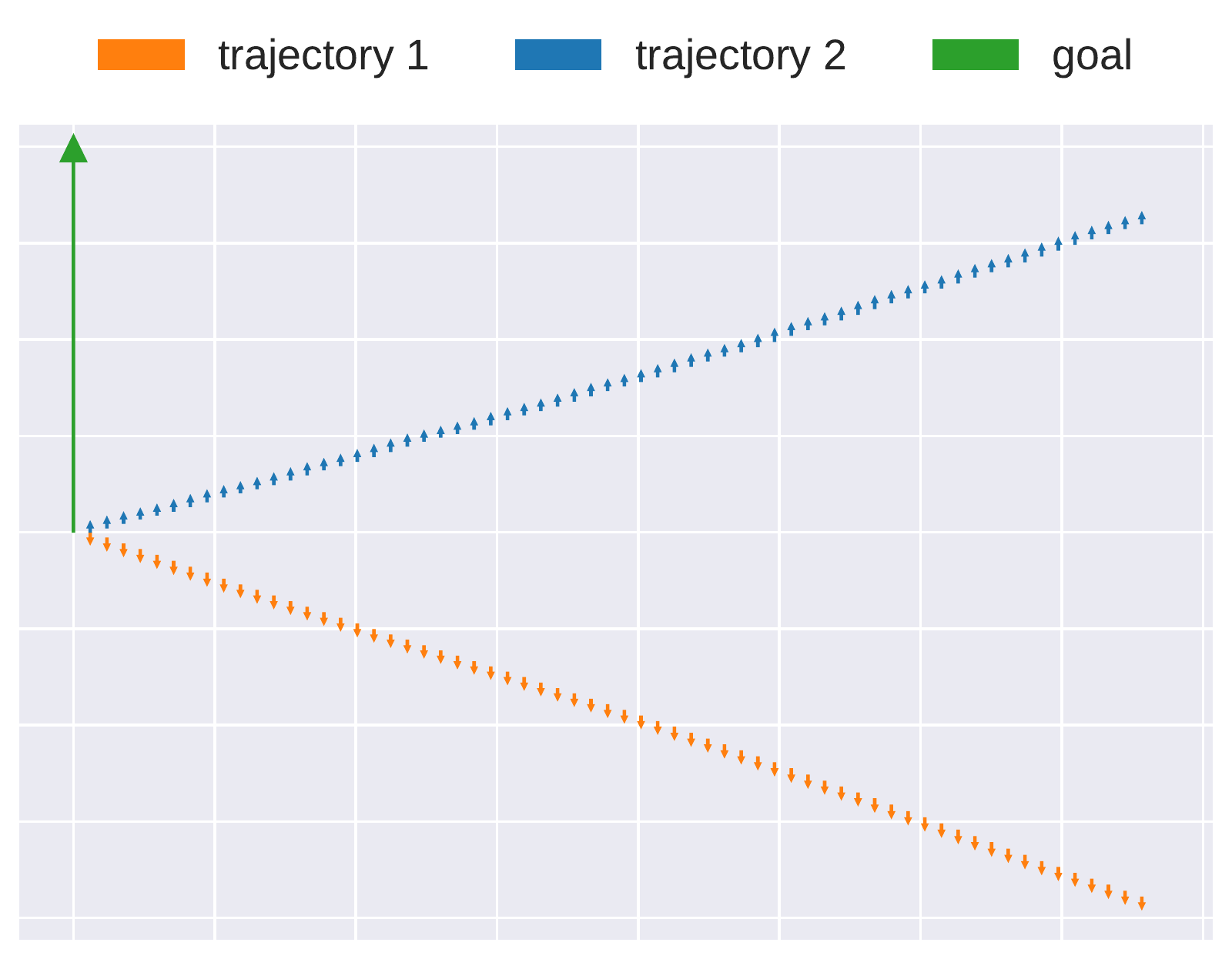}
		\caption{HalfCheetah-Fwd-Back}
	\end{subfigure}
	\begin{subfigure}[b]{0.32\linewidth}
		\centering
		\includegraphics[width=0.98\linewidth]{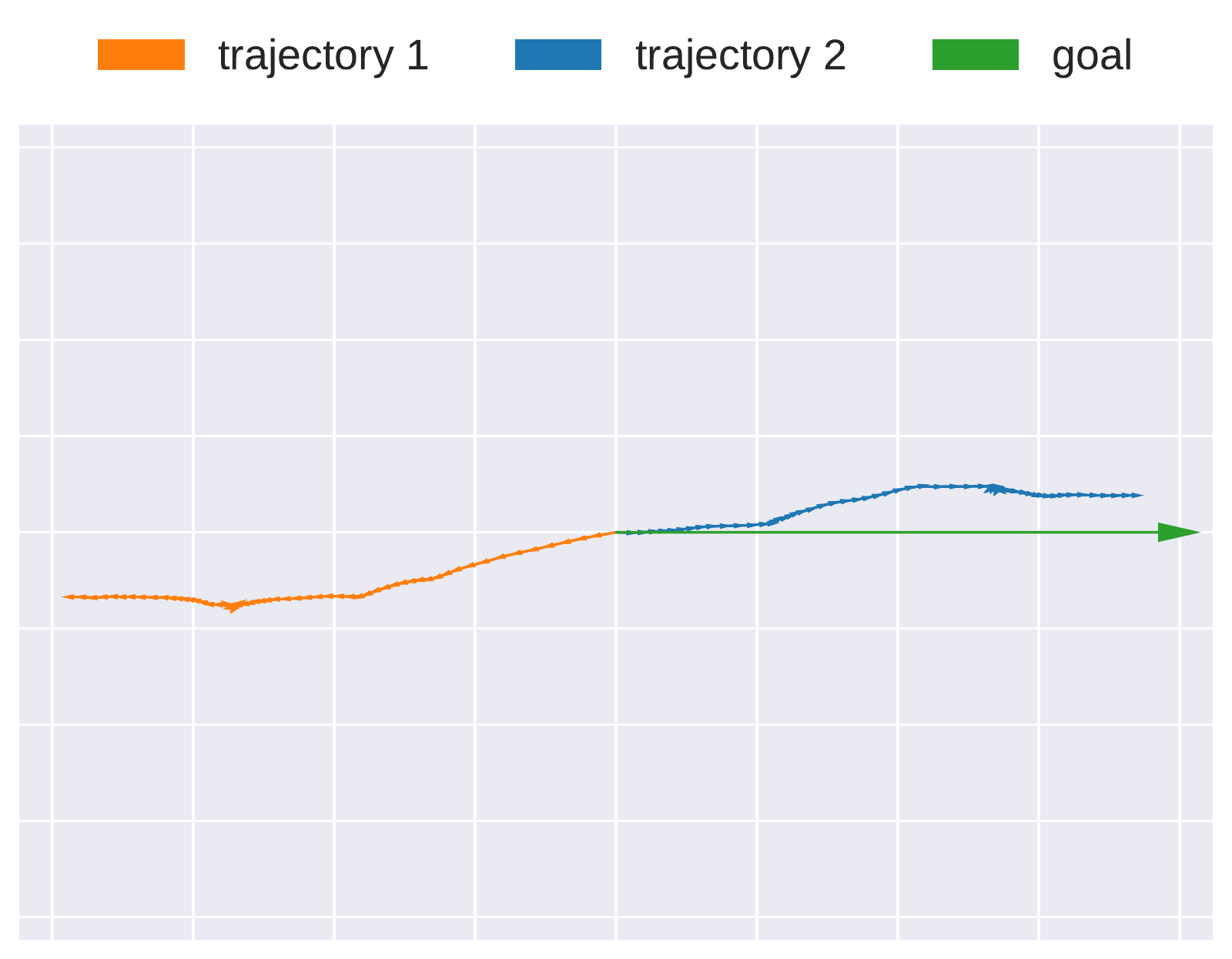}
		\caption{Ant-Fwd-Back}
	\end{subfigure}
	\begin{subfigure}[b]{0.32\linewidth}
		\centering
		\includegraphics[width=0.98\linewidth]{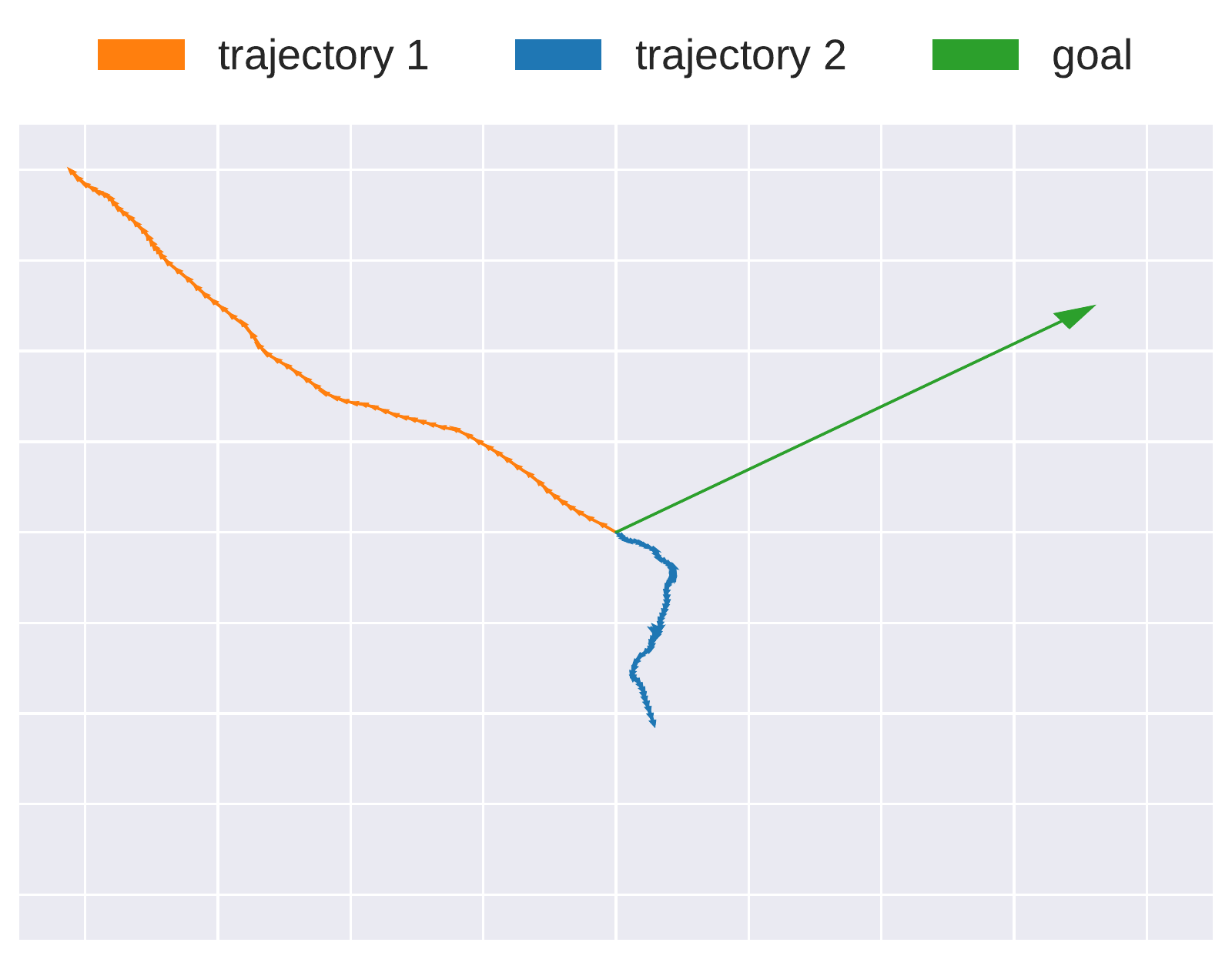}
		\caption{Ant-Rand-Dir}
	\end{subfigure}
\end{figure}

To facilitate human participation, we project the agent trajectory and the goal direction vector to a 2D coordinate system, \ie, extracting the agent location coordinate from the state representation. The human participant watches the query trajectory pair and labels the preference according to the assigned task goal vector. We implement this interface for three tasks: \texttt{HalfCheetah-Fwd-Back}, \texttt{Ant-Fwd-Back}, \texttt{Ant-Rand-Dir}. Since the \texttt{HalfCheetah} agent can only move in a single dimension, we fill in the $x$-axis of \texttt{HalfCheetah-Fwd-Back} interface by the timestep index.

\end{document}